\documentclass[12pt]{elsarticle}
\journal{Artificial Intelligence}

\makeatletter
\def\ps@pprintTitle{%
  \let\@oddhead\@empty
  \let\@evenhead\@empty
  \let\@oddfoot\@empty
  \let\@evenfoot\@oddfoot
}
\makeatother

\usepackage{booktabs} 
\usepackage{amssymb}
\usepackage{amsmath}
\usepackage{wasysym}
\usepackage{epigraph}
\usepackage{graphics}
\usepackage{setspace}
\usepackage{fancybox} 
\newtheorem{definition}{Definition}
\newtheorem{theorem}{Theorem}
\newtheorem{lemma}{Lemma}
\newtheorem{claim}{Claim}
\newtheorem{corollary}{Corollary}

\newenvironment{proof}{\noindent{\sf Proof.}}{\hfill $\boxtimes\hspace{2mm}$\linebreak}
\renewcommand{\phi}{\varphi}
\renewcommand{\epsilon}{\varepsilon}

\renewcommand{\phi}{\varphi}
\renewcommand{\epsilon}{\varepsilon}

\newcommand{\N}{{\sf N}}
\newcommand{\B}{{\sf B}}
\newcommand{\cN}{{\sf \overline{N}}}
\newcommand{\K}{{\sf K}}
\newcommand{\cK}{{\sf \overline{K}}}
\newcommand{\SSS}{{\sf S}}
\newenvironment{proof-of-claim}{\noindent{\sc Proof of Claim.}}{\hfill $\boxtimes\hspace{2mm}$\linebreak}

\newsavebox{\diamonddotsavebox}
\sbox{\diamonddotsavebox}{$\Diamond$\hspace{-1.8mm}\raisebox{0.3mm}{$\cdot$}\hspace{1mm}}

\begin{document}
\begin{frontmatter}

\author{Pavel Naumov}
\address{Claremont McKenna College, Claremont, California, USA}
\ead{pgn2@cornell.edu}
\author{Jia Tao}
\address{Lafayette College, Easton, Pennsylvania, USA}
\ead{taoj@lafayette.edu}







\title{Blameworthiness in Strategic Games\\ with Imperfect Information}
\title{Knowledge and Blameworthiness}

\begin{abstract}
    Blameworthiness of an agent or a coalition of agents is often defined in terms of the principle of alternative possibilities: for the coalition to be responsible for an outcome, the outcome must take place and the coalition should have had a strategy to prevent it. In this article we argue that in the settings with imperfect information, not only should the coalition have had a strategy, but it also should have known that it had a strategy, and it should have known what the strategy was. 
    
    The main technical result of the article is a sound and complete bimodal logic that describes the interplay between knowledge and blameworthiness in strategic games with imperfect information.
\end{abstract}

\end{frontmatter}



\section{Introduction}

In this article we study blameworthiness of agents and their coalitions in multiagent systems. Throughout centuries, blameworthiness, especially in the context of free will and moral responsibility, has been at the focus of philosophical discussions~\cite{se13eb}. These discussions continue in the modern time~\cite{f94p,fr00,nk07nous,m15ps,w17}. Frankfurt~\cite{f69tjop} acknowledges that a dominant role in these discussions has been played by what he calls a {\em principle of alternate possibilities}: ``a person is morally responsible for what he has done only if he could have done otherwise". 
Following the established tradition~\cite{w17}, we refer to this principle as the principle of {\em alternative} possibilities. Cushman~\cite{c15cop} talks about {\em counterfactual possibility}: ``a person could have prevented their harmful conduct, even though they did not".  

Others refer to an alternative possibility as a {\em counterfactual} possibility~\cite{c15cop,h16}.
Halpern and Pearl proposed several versions of a formal definition of causality as a relation between sets of variables that include a counterfactual requirement~\cite{h16}. Halpern and Kleiman-Weiner~\cite{hk18aaai} used a similar setting to define {\em degrees} of blameworthiness. Batusov and Soutchanski~\cite{bs18aaai} gave a counterfactual-based definition of causality in situation calculus. Alechina, Halpern, and Logan~\cite{ahl17aamas} applied counterfactual definition of causality to team plans.
In~\cite{nt19aaai}, we proposed a logical system that describes properties of coalition  blameworthiness in strategic games as a modal operator whose semantics is also based on the principle of alternative possibilities.

Although the principle of alternative possibilities makes sense in the settings with perfect information, it needs to be adjusted for settings with imperfect information. Indeed, consider a traffic situation depicted in Figure~\ref{intro-example-1 figure}. A self-driving truck $t$ and a regular car $c$ are approaching an intersection at which truck $t$ must stop to yield to car $c$. The truck is experiencing a sudden brake failure and it cannot stop, nor can it slow down at the intersection. The truck turns on flashing lights and sends distress signals to other self-driving cars by radio. 
The driver of car $c$ can see the flashing lights, but she does not receive the radio signal. She can also observe that the truck does not slow down. The driver of car $c$ has two potential strategies to avoid a collision with the truck: to slow down or to accelerate. 
\begin{figure}[ht]
\begin{center}
\vspace{0mm}
\scalebox{0.6}{\includegraphics{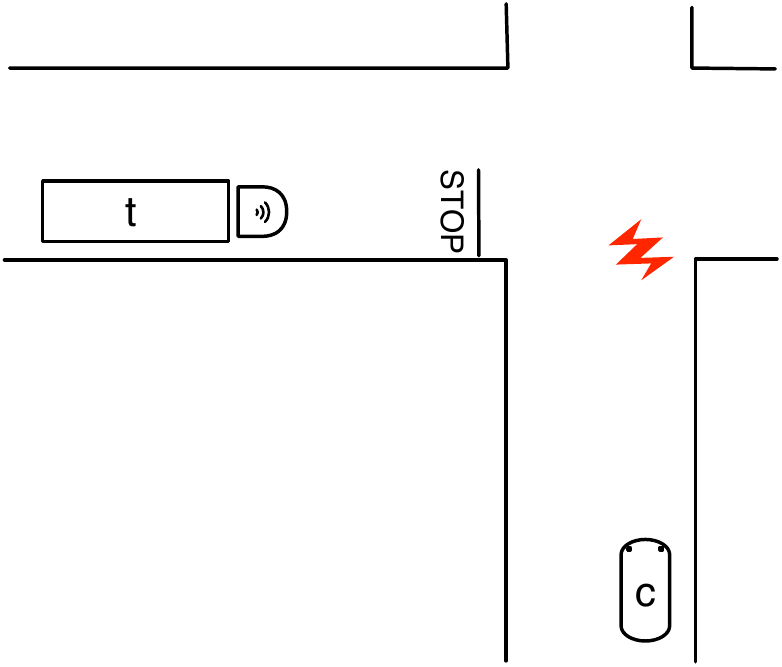}}
\caption{A traffic situation.}\label{intro-example-1 figure}
\vspace{-5mm}
\end{center}
\end{figure}
The driver understands that one of these two strategies will succeed, but since she does not know the exact speed of the truck, she does not know which of the two strategies will succeed. Suppose that the collision could be avoided if the car accelerates, but the car driver decides to slow down. The vehicles collide. According to the principle of alternative possibilities, the driver of the car is responsible for the collision because she had a strategy to avoid the collision but did not use it. 

It is not likely, however, that a court will find the driver of car $c$ responsible for the accident. For example, US Model Penal Code~\cite{ali62} distinguishes different forms of legal liability as different combinations of ``guilty actions'' and ``guilty mind''. The situation in our example falls under strict liability (pure ``guilty actions'' without an accompanied ``guilty mind''). In many situations, strict liability does not lead to legal liability.

In this article we propose a formal semantics of blameworthiness in strategic games with imperfect information. According to this semantics, an agent (or a coalition of agents) is blamable for $\phi$ if $\phi$ is true and the agent {\em knew how} to prevent $\phi$.  
In our example, since the driver of the car does not know that she must accelerate in order to avoid the collision, she cannot be blamed for the collision. We write this as:
$
\neg\B_{c}(\mbox{``Vehicles collided.''}).
$
Now, consider a similar traffic situation in which car $c$ is a self-driving vehicle. The car receives the distress signal from truck $t$, which contains the truck's exact speed. From this information, car $c$ determines that it can avoid the collision if it accelerates. However, if the car slows down, then the vehicles collide and the self-driving car $c$ is blameable for the collision: 
$
\B_{c}(\mbox{``Vehicles collided.''}).
$

The main technical result of this article is a bimodal logical system that describes the interplay between  knowledge and blameworthiness of coalitions in strategic games with imperfect information.

The article is organized as follows. In the next section we review the literature. Section~\ref{syntax and semantics section} presents the formal syntax and semantics of our logical system. Section~\ref{axioms section} introduces our axioms and compares them to those in the related works. Section~\ref{examples section} gives examples of formal derivations in the proposed logical system. Sections~\ref{soundness section} and \ref{completeness section} prove the soundness and the completeness of our system. Section~\ref{conclusion section} concludes with a discussion of future work.

\section{Related Literature}

Although the study of responsibility and blameworthiness has a long history in philosophy, the use of formal logical systems to capture these notions is a recent development. Xu~\cite{x98jpl} proposed a complete logical system for reasoning about responsibility of individual agents in multiagent systems. His approach was extended to coalitions by Broersen, Herzig, and  Troquard~\cite{bht09jancl}. The definition of responsibility in these works is different from ours. They assume that an agent or a coalition of agents is responsible for an outcome if the actions that they took unavoidably lead to the outcome. Xu~\cite{x98jpl} also requires a possibility that the outcome might not happen. However, he does not require that the agent has a strategy to prevent the outcome. Thus, their definitions are not based on the principle of alternative possibilities.

Halpern and Pearl gave several versions of a formal definition of causality between sets of variables using counterfactuals~\cite{h16}. Lorini and Schwarzentruber~\cite{ls11ai} observed that a variation of this definition can be captured in STIT logic~\cite{bp90krdr,h01,h95jpl,hp17rsl,ow16sl}. 
They said that there is a counterfactual dependence between actions of a coalition $C$ and an outcome $\phi$ if $\phi$ is true and the complement of the coalition $C$ had no strategy to force $\phi$. In their notations: 
$
{\sf CHP}_C\phi\equiv \phi\wedge \neg[\mathcal{A}\!\setminus\! C]\,\phi,
$
where $\mathcal{A}$ is the set of all agents.
They also observed that many human emotions (regret, rejoice, disappointment, elation) can be expressed through a combination of the modality ${\sf CHP}$ and the knowledge modality.


The game-like setting of this article closely resembles the semantics of Mark Pauly's logic of coalition power~\cite{p01illc,p02}. His approach has been  widely investigated in the literature~\cite{g01tark,vw05ai,b07ijcai,sgvw06aamas,abvs10jal,avw09ai,b14sr,gjt13jaamas,alnr11jlc,ga17tark,ge18aamas,nr18kr}. Logics of coalition power study modality that express what a coalition {\em can do}. In~\cite{nt19aaai} we modified Mark Pauly's semantics to express what a coalition {\em could have done}. We axiomatized a logic that combines statements ``$\phi$ is true'' and ``coalition $C$ could have prevented $\phi$'' into a single modality $\B_C\phi$.

In this article we replace ``coalition $C$ could have prevented $\phi$'' in~\cite{nt19aaai} with ``coalition $C$ knew how it could have prevented $\phi$''. 
The distinction between an agent having a strategy, knowing that a strategy exists, and knowing what the strategy is has been studied before.  While Jamroga and {\AA}gotnes~\cite{ja07jancl} talked about ``knowledge to identify and execute a strategy",  Jamroga and van der Hoek~\cite{jv04fm} discussed ``difference between an agent knowing that he has a suitable strategy and knowing the strategy itself". Van Benthem~\cite{v01ber} called such strategies ``uniform". Broersen~\cite{b08deon} talked about ``knowingly doing'', while  Broersen, Herzig, and Troquard~\cite{bht09jancl} discussed modality ``know they can do''. We used term ``executable strategy"~\cite{nt17aamas}. Wang~\cite{w15lori,w17synthese} talked about ``knowing how". 

The properties of know-how as a modality have been previously axiomatized in different settings. {\AA}got\-nes and Alechina introduced a complete axiomatization of an interplay between single-agent knowledge and coalition know-how modalities to achieve a goal in one step~\cite{aa16jlc}. A modal logic that combines the distributed knowledge modality with the coalition know-how modality to maintain a goal was axiomatized by us in \cite{nt17aamas}. A sound and complete logical system in a single-agent setting for know-how strategies to {achieve} a goal in multiple steps rather than to {maintain} a goal is developed by Fervari, Herzig, Li, and Wang~\cite{fhlw17ijcai}. In \cite{nt17tark,nt18ai}, we developed a trimodal logical system that describes an interplay between the (not know-how) coalition strategic modality, the coalition know-how modality, and the distributed knowledge modality. In \cite{nt18aaai}, we proposed a logical system that combines the coalition know-how modality with the distributed knowledge modality in the perfect recall setting. In \cite{nt18aamas}, we introduced a logical system for second-order know-how. Wang proposed a complete axiomatization of ``knowing how'' as a binary modality~\cite{w15lori,w17synthese}, but his logical system does not include the knowledge modality.

The axioms of the logical system proposed in this article are very similar to our axioms in~\cite{nt19aaai} for blameworthiness in games with perfect information and so are the proofs of soundness of these axioms. The most important contribution of this article is the proof of completeness, in which the construction from~\cite{nt19aaai} is significantly modified to incorporate distributed knowledge. These modifications are discussed in the beginning of Section~\ref{completeness section}. 


\section{Syntax and Semantics}\label{syntax and semantics section}

In this article we assume a fixed set $\mathcal{A}$ of agents and a fixed set of propositional variables. By a coalition we mean an arbitrary subset of set $\mathcal{A}$.
\begin{definition}\label{Phi}
$\Phi$ is the minimal set of formulae such that
\begin{enumerate}
    \item $p\in\Phi$ for each propositional variable $p$,
    \item $\phi\to\psi,\neg\phi\in\Phi$ for all formulae $\phi,\psi\in\Phi$,
    \item $\K_C\phi$, $\B_C\phi\in\Phi$ for each coalition $C\subseteq\mathcal{A}$ and each $\phi\in\Phi$. 
\end{enumerate}
\end{definition}
In other words, language $\Phi$ is defined by  grammar:
$$
\phi := p\;|\;\neg\phi\;|\;\phi\to\phi\;|\;\K_C\phi\;|\;\B_C\phi.
$$
Formula $\K_C\phi$ is read as ``coalition $C$ distributively knew before the actions were taken that statement $\phi$ would be true'' and formula $\B_C\phi$ as ``coalition $C$ is blamable for $\phi$''.

Boolean connectives $\vee$, $\wedge$, and $\leftrightarrow$ as well as constants $\bot$ and $\top$ are defined in the standard way. By formula $\cK_C\phi$ we mean $\neg\K_C\neg\phi$. For the disjunction of multiple formulae, we assume that parentheses are nested to the left. That is, formula $\chi_1\vee\chi_2\vee\chi_3$ is a shorthand for $(\chi_1\vee\chi_2)\vee\chi_3$. As usual, the empty disjunction is defined to be $\bot$. For any two sets $X$ and $Y$, by $X^Y$ we denote the set of all functions from $Y$ to $X$.

The formal semantics of modalities $\K$ and  $\B$ is defined in terms of models, which we call \emph{games}. These are one-shot strategic games with imperfect information. We specify the set of actions by all agents, or a {\em complete action profile}, as a function $\delta\in \Delta^\mathcal{A}$ from the set of all agents $\mathcal{A}$ to the set of all actions $\Delta$.

\begin{definition}\label{game definition}
A game is a tuple $\left(I, \{\sim_a\}_{a\in\mathcal{A}},\Delta,\Omega,P,\pi\right)$, where 
\begin{enumerate}
    \item $I$ is a set of ``initial states'',
    \item $\sim_a$ is an ``indistinguishability'' equivalence relation on set $I$,
    \item $\Delta$ is a nonempty set of ``actions'',
    \item $\Omega$ is a set of ``outcomes'',
    \item the set of ``plays'' $P$ is an arbitrary set of tuples $(\alpha,\delta,\omega)\in I\times \Delta^\mathcal{A}\times \Omega$ where for each initial state $\alpha\in I$ and each complete action profile $\delta\in\Delta^\mathcal{A}$, there is at least one outcome $\omega\in \Omega$ such that $(\alpha,\delta,\omega)\in P$,
    
    \item $\pi$ is a function that maps propositional variables into subsets of $P$.
\end{enumerate}
\end{definition} 

In the introductory example, the set $I$ has two states {\em high} and {\em low}, corresponding to the truck going at a high or low speed, respectively. The driver of the regular car $c$ cannot distinguish these two states while these states can be distinguished by a self-driving version of car $c$. For the sake of simplicity, assume that there are two actions that car $c$ can take: $\Delta=\{\mbox{\em slow-down}, \mbox{\em speed-up}\}$ and two possible outcomes: $\Omega=\{\mbox{\em collision}, \mbox{\em no collision}\}$. Vehicles collide if either the truck goes with a low speed and the car decides to slow-down or the truck goes with a high speed and the car decides to accelerate. In our case there is only one agent (car $c$), so the complete action profile can be described by giving just the action of this agent. We refer to the two complete action profiles in this situation simply as profile {\em slow-down} and profile {\em speed-up}. 
The list of all possible scenarios (or ``plays'') is given by the set 
\begin{eqnarray*}
P&=&\{(\mbox{\em high},\mbox{\em speed-up},\mbox{\em collision}),
(\mbox{\em high},\mbox{\em slow-down},\mbox{\em no collision}),\\
&&\{(\mbox{\em low},\mbox{\em speed-up},\mbox{\em no collision}),
(\mbox{\em low},\mbox{\em slow-down},\mbox{\em collision})
\}.
\end{eqnarray*}
Note that in our example an initial state and an action profile uniquely determine the outcome. In general, just like in~\cite{nt19aaai}, we allow nondeterministic games where this does not have to be true. However, unlike~\cite{nt19aaai}, we do require that for each initial state and each action profile there is at least one outcome. As we discuss in Section~\ref{axioms section}, this requirement captures better the intuitive notion of blameworthiness. 

Whether statement $\B_C\phi$ is true or false depends not only on the outcome but also on the initial state of the game. Indeed, coalition $C$ might have known how to prevent $\phi$ in one initial state but not in the other. For this reason, we assume that all statements are true or false for a particular play of the game.  For example, propositional variable $p$ can stand for ``car $c$ slowed down and collided with truck $t$ going at a high speed''.  As a result, function $\pi$ in the definition above maps $p$ into subsets of $P$ rather than subsets of $\Omega$.

By an action profile of a coalition $C$ we mean an arbitrary function $s\in \Delta^C$ that assigns an action to each member of the coalition. If $s_1$ and $s_2$ are action profiles of coalitions $C_1$ and $C_2$, respectively, and $C$ is any coalition such that $C\subseteq C_1\cap C_2$, then we write $s_1=_C s_2$ to denote that $s_1(a)=s_2(a)$ for each agent $a\in C$. We write $\alpha\sim_C\alpha'$ if $\alpha\sim_a\alpha'$ for each $a\in C$. In particular, it means that $\alpha\sim_\varnothing\alpha'$ for any two initial states $\alpha,\alpha'\in I$.

Next is the key definition of this article. Its item 5 formally specifies blameworthiness using the principle of alternative possibilities.  In order for a coalition to be blamable for $\phi$, not only must $\phi$ be true and the coalition should have had a strategy to prevent $\phi$, but this strategy should work in all initial states that the coalition cannot distinguish from the current state. In other words, the coalition should have known the strategy.

\begin{definition}\label{sat} 
For any game $\left(I, \{\sim_a\}_{a\in\mathcal{A}},\Delta,\Omega,P,\pi\right)$, any formula $\phi\in\Phi$, and any play $(\alpha,\delta,\omega)\in P$, 
the satisfiability relation $(\alpha,\delta,\omega)\Vdash\phi$ is defined recursively as follows:
\begin{enumerate}
    \item $(\alpha,\delta,\omega)\Vdash p$ if $(\alpha,\delta,\omega)\in \pi(p)$, where $p$ is a propositional variable,
    \item $(\alpha,\delta,\omega)\Vdash \neg\phi$ if $(\alpha,\delta,\omega)\nVdash \phi$,
    \item $(\alpha,\delta,\omega)\Vdash\phi\to\psi$ if $(\alpha,\delta,\omega)\nVdash\phi$ or $(\alpha,\delta,\omega)\Vdash\psi$,
    \item $(\alpha,\delta,\omega)\Vdash\K_C\phi$ if $(\alpha',\delta',\omega')\Vdash\phi$ for each play $(\alpha',\delta',\omega')\in P$ such that $\alpha\sim_C\alpha'$,
    \item $(\alpha,\delta,\omega)\Vdash\B_C\phi$ if $(\alpha,\delta,\omega)\Vdash\phi$ and there is an action profile $s\in \Delta^C$ of coalition $C$ such that for each play $(\alpha',\delta',\omega')\in P$, if $\alpha\sim_C\alpha'$ and $s=_C\delta'$, then $(\alpha',\delta',\omega')\nVdash\phi$.
\end{enumerate}
\end{definition}
Since modality $\K_C$ represents {\em a priori} (before the actions) knowledge of coalition $C$, only the initial states in plays $(\alpha,\delta,\omega)$ and $(\alpha',\delta',\omega')$ are indistinguishable in item 4 of Definition~\ref{sat}.
Similarly, since item 5 of the above definition refers to indistinguishability relation $\sim_C$ on initial states, not outcomes, the knowledge of the strategy to prevent captured by the modality $\B_C\phi$ is also {\em a priori} knowledge of coalition $C$.

For formula $\B_C\phi$ to be true, item 5 of Definition~\ref{sat} requires coalition $C$ to know a strategy to prevent $\phi$, but it does not require the coalition $C$ to know that $\phi$ is true. This captures a common belief, for example, that a murder is blameable for a death even if the murder does not know that the victim died. 

Note that in item 5 of the above definition we do not assume that coalition $C$ is a minimal one that knew how to prevent the outcome. This is different from the definition of blameworthiness in~\cite{h17}. Our approach is consistent with how word ``blame'' is often used in English. For example, the sentence ``Millennials being blamed for decline of American cheese''~\cite{g18foxnews} does not imply that no one in the millennial generation likes American cheese.


\section{Axioms}\label{axioms section}

In addition to the propositional tautologies in  language $\Phi$, our logical system contains the following axioms.

\begin{enumerate}
    \item Truth: $\K_C\phi\to\phi$ and $\B_C\phi\to\phi$,
    \item Distributivity: $\K_C(\phi\to\psi)\to(\K_C\phi\to \K_C\psi)$,
    \item Negative Introspection: $\neg\K_C\phi\to\K_C\neg\K_C\phi$,
    \item Monotonicity: $\K_C\phi\to\K_D\phi$ and $\B_C\phi\to\B_D\phi$, where $C\subseteq D$,
    \item None to Blame: $\neg\B_\varnothing\phi$,
    \item Blamelessness of Truth: $\neg\B_C\top$,
    \item Joint Responsibility:  $\cK_C\B_C\phi\wedge\cK_D\B_D\psi\to (\phi\vee\psi\to\B_{C\cup D}(\phi\vee\psi))$, where $C\cap D=\varnothing$,
    \item Blame for Known Cause: $\K_C(\phi\to\psi)\to(\B_C\psi\to(\phi\to \B_C\phi))$,
    \item Knowledge of Fairness: $\B_C\phi\to\K_C(\phi\to\B_C\phi)$.
    
\end{enumerate}
We write $\vdash\phi$ if formula $\phi$ is provable from the axioms of our system using the Modus Ponens and
the Necessitation inference rules:
$$
\dfrac{\phi,\phi\to\psi}{\psi},
\hspace{20mm}
\dfrac{\phi}{\K_C\phi}.
$$
We write $X\vdash\phi$ if formula $\phi\in\Phi$ is provable from the theorems of our logical system and an additional set of axioms $X$ using only the Modus Ponens inference rule. Note that if set $X$ is empty, then statement $X\vdash\phi$ is equivalent to $\vdash\phi$. We say that set $X$ is consistent if $X\nvdash\bot$.

The Truth, the Distributivity, the Negative Introspection, and the Monotonicity axioms for epistemic modality $\K$ are the standard S5 axioms from the logic of distributed knowledge. The Truth axiom for blameworthiness modality $\B$ states that a coalition could only be blamed for something true. The Monotonicity axiom for the  blameworthiness modality states that if a part of a coalition is blamable for something, then the whole coalition is also blamable for the same thing. The None to Blame axiom says that an empty coalition can be blamed for nothing. The Blamelessness of Truth axiom states that no coalition can be blamed for a tautology. This is a new axiom that does not have an equivalent in~\cite{nt19aaai}. The soundness of this axiom relies on our assumption in item 4 of Definition~\ref{game definition} that any combination of an initial state and a complete action profile has at least one outcome. Without this assumption, a coalition $C$ might be able to terminate the game without reaching an outcome. In other words, coalition $C$ might have a strategy to ``prevent'' $\top$.

The remaining three axioms describe the interplay between knowledge and blameworthiness modalities.
The Joint Responsibility axiom says that if a coalition $C$ cannot exclude a possibility of being blamable for $\phi$, a coalition $D$ cannot exclude a possibility of being blamable for $\psi$, and the disjunction $\phi\vee\psi$ is true, then the joint coalition $C\cup D$ is blamable for the disjunction. This axiom resembles Xu's axiom for the independence of individual agents~\cite{x98jpl},
$$
\cN\B_{a_1}\phi_1\wedge\dots\wedge\cN\B_{a_n}\phi_n\to \cN(\B_{a_1}\phi_1\wedge\dots\wedge\B_{a_n}\phi_n),
$$ where modality $\cN$ is an abbreviation for $\neg\N\neg$ and formula $\N\phi$ stands for ``formula $\phi$ is universally true in the given model''. 
Broersen, Herzig, and Troquard~\cite{bht09jancl} captured the independence of disjoint coalitions $C$ and $D$ in their Lemma 17:
$$
\cN\B_C\phi\wedge\cN\B_D\psi\to\cN(\B_C\phi\wedge\B_D\psi).
$$
In spite of certain similarity, the definition of responsibility used in \cite{x98jpl} and \cite{bht09jancl} does not assume the principle of alternative possibilities.
The Joint Responsibility axiom is also similar to Marc Pauly's Cooperation axiom for the logic of coalitional power~\cite{p01illc,p02}:
$$
\SSS_C\phi\wedge\SSS_D\psi\to\SSS_{C\cup D}(\phi\wedge\psi),
$$
where coalitions $C$ and $D$ are disjoint and $\SSS_C\phi$ stands for ``coalition $C$ has a strategy to achieve $\phi$". Finally, The Joint Responsibility axiom in this article is a generalization of the Joint Responsibility axiom for games with perfect information~\cite{nt19aaai}:
$$\cN\B_C\phi\wedge\cN\B_D\psi\to (\phi\vee\psi\to\B_{C\cup D}(\phi\vee\psi)),$$
where coalitions $C$ and $D$ are disjoint.

Informally, if $\K_C(\phi\to\psi)$, then we say that $\phi$ is a cause of $\psi$ known to coalition $C$. Note that if a coalition has a strategy to prevent a known cause $\phi$, then the coalition also has a strategy to prevent $\psi$. However, it is not true that the coalition $C$ should be blamed for $\phi$ if it can be blamed for $\psi$ because ``the known cause'' $\phi$ might not be true. If the known cause $\phi$ is true, then the blameworthiness for $\psi$ implies the blameworthiness for $\phi$. This is captured in the Blame for Known Cause axiom. A similar axiom, but without knowledge, appeared in~\cite{nt19aaai}.



Our last axiom also goes back to one of the axioms for the games with perfect information. The Fairness axiom for these games
$$
\B_C\phi\to\N(\phi\to\B_C\phi)
$$
states ``if a coalition $C$ is blamed for $\phi$, then it should be blamed for $\phi$ whenever $\phi$ is true''~\cite{nt19aaai}. The Knowledge of Fairness axiom in the current article states that  if a coalition $C$ is blamable for $\phi$ in an imperfect information game, then it {\em knows} that it is blamable for $\phi$ whenever $\phi$ is true.

Next, we state the deduction and Lindenbaum  lemmas for our logical system. These lemmas are used later in the proof of the completeness.

\begin{lemma}[deduction]\label{deduction lemma}
If $X,\phi\vdash\psi$, then $X\vdash\phi\to\psi$.
\end{lemma}
\begin{proof}
Suppose that sequence $\psi_1,\dots,\psi_n$ is a proof from set $X\cup\{\phi\}$ and the theorems of our logical system that uses the Modus Ponens inference rule only. In other words, for each $k\le n$, either
\begin{enumerate}
    \item $\vdash\psi_k$, or
    \item $\psi_k\in X$, or
    \item $\psi_k$ is equal to $\phi$, or
    \item there are $i,j<k$ such that formula $\psi_j$ is equal to $\psi_i\to\psi_k$.
\end{enumerate}
It suffices to show that $X,\phi\vdash\psi_k$ for each $k\le n$. We prove this by induction on $k$ through considering the four cases above separately.

\vspace{1mm}
\noindent{\bf Case 1}: $\vdash\psi_k$. Note that $\psi_k\to(\phi\to\psi_k)$ is a propositional tautology, and thus, is an axiom of our logical system. Hence, $\vdash\phi\to\psi_k$ by the Modus Ponens inference rule. Therefore, $X\vdash\phi\to\psi_k$. 

\vspace{1mm}
\noindent{\bf Case 2}: $\psi_k\in X$. Then, $X\vdash\psi_k$.

\vspace{1mm}
\noindent{\bf Case 3}: formula $\psi_k$ is equal to $\phi$. Thus, $\phi\to\psi_k$ is a propositional tautology. Therefore, $X\vdash\phi\to\psi_k$. 

\vspace{1mm}
\noindent{\bf Case 4}:  formula $\psi_j$ is equal to $\psi_i\to\psi_k$ for some $i,j<k$. Thus, by the induction hypothesis, $X\vdash\phi\to\psi_i$ and $X\vdash\phi\to(\psi_i\to\psi_k)$. Note that formula 
$
(\phi\to\psi_i)\to((\phi\to(\psi_i\to\psi_k))\to(\phi\to\psi_k))
$
is a propositional tautology. Therefore, $X\vdash \phi\to\psi_k$ by applying the Modus Ponens inference rule twice.
\end{proof}

Note that it is important for the above proof that $X\vdash\phi$ stands for derivability only using the Modus Ponens inference rule. For example, if the Necessitation inference rule is allowed, then the proof will have to include one more case where $\psi_k$ is formula $\K_C\psi_i$ for some coalition $C\subseteq\mathcal{A}$, and some integer $i< k$. In this case we will need to prove that if $X\vdash \phi\to\psi_i$, then $X\vdash \phi\to\K_C\psi_i$, which is not true.

\begin{lemma}[Lindenbaum]\label{Lindenbaum's lemma}
Any consistent set of formulae can be extended to a maximal consistent set of formulae.
\end{lemma}
\begin{proof}
The standard proof of Lindenbaum's lemma applies here~\cite[Proposition 2.14]{m09}. 
\end{proof}

\section{Examples of Derivations}\label{examples section}


We prove the soundness of the axioms of our logical system in the next section. Here we prove several lemmas about our formal system that are used later in the proof of the completeness. 



\begin{lemma}\label{alt fairness lemma}
$\vdash\cK_C\B_C\phi\to(\phi\to\B_C\phi)$.
\end{lemma}
\begin{proof}
Note that $\vdash \B_C\phi\to\K_C(\phi\to\B_C\phi)$ by the Knowledge of Fairness axiom. Thus, $\vdash \neg\K_C(\phi\to\B_C\phi)\to \neg\B_C\phi$, by the law of contrapositive. Then, $\vdash \K_C(\neg\K_C(\phi\to\B_C\phi)\to \neg\B_C\phi)$ by the Necessitation inference rule. Hence, by the Distributivity axiom and the Modus Ponens inference rule,
$$\vdash \K_C\neg\K_C(\phi\to\B_C\phi)\to \K_C\neg\B_C\phi.$$ 
At the same time, by the Negative Introspection axiom:
$$
\vdash \neg\K_C(\phi\to\B_C\phi)\to\K_C\neg\K_C(\phi\to\B_C\phi).
$$
Then, by the laws of propositional reasoning,
$$\vdash \neg\K_C(\phi\to\B_C\phi)\to \K_C\neg\B_C\phi.$$
Thus, by the law of contrapositive,
$$\vdash \neg\K_C\neg\B_C\phi\to \K_C(\phi\to\B_C\phi).$$
Since $\K_C(\phi\to\B_C\phi)\to(\phi\to\B_C\phi)$ is an instance of the Truth axiom, by propositional reasoning,
$$\vdash \neg\K_C\neg\B_C\phi\to (\phi\to\B_C\phi).$$
Therefore, $\vdash \cK_C\B_C\phi\to (\phi\to\B_C\phi)$ by the definition of $\cK_C$.
\end{proof}

\begin{lemma}\label{alt cause lemma}
If $\vdash \phi\leftrightarrow \psi$, then $\vdash \B_C\phi\to\B_C\psi$.
\end{lemma}
\begin{proof}
Assumption  $\vdash \phi\leftrightarrow \psi$ implies $\vdash \psi\to \phi$ by the laws of propositional reasoning. Hence, $\vdash \K_C(\psi\to \phi)$ by the Necessitation inference rule. Thus, $\vdash \B_C\phi\to(\psi\to \B_C\psi)$ 
by the Blame for Known Cause axiom and the Modus Ponens inference rule. Hence, $\vdash \psi\to(\B_C\phi\to \B_C\psi)$ by propositional reasoning. 
Then, again by propositional reasoning,
\begin{equation}\label{sofia}
\vdash (\B_C\phi\to\psi)\to (\B_C\phi\to \B_C\psi).
\end{equation}
Observe that $\vdash \B_C\phi\to\phi$ by the Truth axiom. Also, $\vdash \phi\leftrightarrow \psi$ by the assumption of the lemma. Then, by the laws of propositional reasoning, $\vdash \B_C\phi\to\psi$. Therefore,
$
\vdash \B_C\phi\to \B_C\psi
$
by the Modus Ponens inference rule from statement~(\ref{sofia}).
\end{proof}

The next lemma states a well-known S5 principle that we use several times in the proofs that follow.

\begin{lemma}\label{add cK lemma}
$\phi\vdash \cK_C\phi$.
\end{lemma}
\begin{proof}
By the Truth axioms, $\vdash\K_C\neg\phi\to\neg\phi$. Hence, by the law of contrapositive, $\vdash\phi\to \neg\K_C\neg\phi$. Thus, $\vdash\phi\to \cK_C\phi$ by the definition of the modality $\cK_C$. Therefore, $\phi\vdash \cK_C\phi$ by the Modus Ponens inference rule.
\end{proof}

The next lemma generalizes the Joint Responsibility axiom from two coalitions to multiple coalitions. Informally, it says that if the disjunction $\chi_1\vee\dots\vee\chi_n$ is true and each of the disjoint coalitions $D_1,\dots,D_n$
cannot exclude a possibility of being blamed for the corresponding disjunct, then together they should be blamed for the disjunction.

\begin{lemma}\label{super joint responsibility lemma}
For any integer $n\ge 0$ and any pairwise disjoint sets $D_1,\dots,D_n$,
$$
\{\cK_{D_i}\B_{D_i}\chi_i\}_{i=1}^n,\;\chi_1\vee\dots\vee\chi_n
\vdash \B_{D_1\cup\dots\cup D_n}(\chi_1\vee \dots\vee\chi_n).
$$
\end{lemma}

\begin{proof}
We prove the lemma by induction on $n$. If $n=0$, then disjunction $\chi_1\vee\dots\vee \chi_n$ is Boolean constant false $\bot$. Hence, the statement of the lemma, $\bot\vdash\B_\varnothing\bot$, is provable in the propositional logic.

Next, assume that $n=1$. Then, from Lemma~\ref{alt fairness lemma} using Modus Ponens rule twice, we get
$\cK_{D_1}\B_{D_1}\chi_1,\chi_1\vdash\B_{D_1}\chi_1$ .

Assume that $n\ge 2$. By the assumption of the lemma that sets $D_1,\dots,D_n$ are pairwise disjoint, the Joint Responsibility axiom, and the Modus Ponens inference rule,
\begin{eqnarray*}
&&\hspace{-8mm}\cK_{D_1\cup \dots \cup D_{n-1}}\B_{D_1\cup \dots \cup D_{n-1}}(\chi_1\vee\dots\vee\chi_{n-1}),
\cK_{D_n}\B_{D_n}\chi_n,\;\chi_1\vee\dots\vee\chi_{n-1}\vee\chi_n\\
&&\hspace{0mm}\vdash \B_{D_1\cup \dots \cup D_{n-1}\cup D_n}(\chi_1\vee\dots\vee\chi_{n-1}\vee \chi_n).
\end{eqnarray*}
Hence, by Lemma~\ref{add cK lemma},
\begin{eqnarray*}
&&\hspace{-8mm}\B_{D_1\cup \dots \cup D_{n-1}}(\chi_1\vee\dots\vee\chi_{n-1}),
\cK_{D_n}\B_{D_n}\chi_n,\;\chi_1\vee\dots\vee\chi_{n-1}\vee\chi_n\\
&&\vdash \B_{D_1\cup \dots \cup D_{n-1}\cup D_n}(\chi_1\vee\dots\vee\chi_{n-1}\vee \chi_n).
\end{eqnarray*}
At the same time, by the induction hypothesis,
$$\{\cK_{D_i}\B_{D_i}\chi_i\}_{i=1}^{n-1},\;\chi_1\vee\dots\vee\chi_{n-1}
\vdash \B_{D_1\cup\dots\cup D_{n-1}}(\chi_1\vee \dots\vee\chi_{n-1}).$$
Thus,
\begin{eqnarray*}
&&\hspace{-5mm}\{\cK_{D_i}\B_{D_i}\chi_i\}_{i=1}^n,\;\chi_1\vee\dots\vee\chi_{n-1},\;\chi_1\vee\dots\vee\chi_{n-1}\vee\chi_n\\
&&\vdash \B_{D_1\cup\dots\cup D_{n-1}\cup D_n}(\chi_1\vee \dots\vee\chi_{n-1}\vee\chi_n).
\end{eqnarray*}
Note that $\chi_1\vee\dots\vee\chi_{n-1}\vdash\chi_1\vee\dots\vee\chi_{n-1}\vee\chi_n$ is provable in the propositional logic. Thus,
\begin{eqnarray}
&&\hspace{-15mm}\{\cK_{D_i}\B_{D_i}\chi_i\}_{i=1}^n,\;\chi_1\vee\dots\vee\chi_{n-1}
 \vdash \B_{D_1\cup\dots\cup D_{n-1}\cup D_n}(\chi_1\vee \dots\vee\chi_{n-1}\vee\chi_n).\label{part 1}
\end{eqnarray}
Similarly, by the Joint Responsibility axiom and the Modus Ponens inference rule,
\begin{eqnarray*}
&&\hspace{-8mm}\cK_{D_1}\B_{D_1}\chi_1,\cK_{D_2\cup \dots \cup D_n}\B_{D_2\cup \dots \cup D_n}(\chi_2\vee\dots\vee\chi_n),\;\chi_1\vee(\chi_2\vee\dots\vee\chi_n)\\
&&\hspace{-0mm}\vdash \B_{D_1\cup \dots \cup D_{n-1}\cup D_n}(\chi_1\vee(\chi_2\vee\dots\vee \chi_n)).
\end{eqnarray*}
Because formula 
$\chi_1\vee(\chi_2\vee\dots\vee \chi_n)\leftrightarrow \chi_1\vee\chi_2\vee\dots\vee \chi_n$ is provable in the propositional logic, by Lemma~\ref{alt cause lemma},
\begin{eqnarray*}
&&\hspace{-7mm}\cK_{D_1}\B_{D_1}\chi_1,\;\cK_{D_2\cup \dots \cup D_n}\B_{D_2\cup \dots \cup D_n}(\chi_2\vee\dots\vee\chi_n),\;\chi_1\vee\chi_2\vee\dots\vee\chi_n\\
&&\hspace{0mm}\vdash \B_{D_1\cup \dots \cup D_{n-1}\cup D_n}(\chi_1\vee\chi_2\vee\dots\vee \chi_n).
\end{eqnarray*}
Hence, by Lemma~\ref{add cK lemma},
\begin{eqnarray*}
&&\hspace{-7mm}\cK_{D_1}\B_{D_1}\chi_1,\;\B_{D_2\cup \dots \cup D_n}(\chi_2\vee\dots\vee\chi_n),\;\chi_1\vee\chi_2\vee\dots\vee\chi_n\\
&&\vdash \B_{D_1\cup \dots \cup D_{n-1}\cup D_n}(\chi_1\vee\chi_2\vee\dots\vee \chi_n).
\end{eqnarray*}
At the same time, by the induction hypothesis,
$$
\{\cK_{D_i}\B_{D_i}\chi_i\}_{i=2}^n,\;\chi_2\vee\dots\vee\chi_n
\vdash \B_{D_2\cup\dots\cup D_n}(\chi_2\vee \dots\vee\chi_n).
$$
Thus,
\begin{eqnarray*}
&&\hspace{-8mm}\{\cK_{D_i}\B_{D_i}\chi_i\}_{i=1}^n,\;\chi_2\vee\dots\vee\chi_n,\;\chi_1\vee\chi_2\vee\dots\vee\chi_n\\
&&\vdash \B_{D_1\cup D_2\cup\dots\cup D_n}(\chi_1\vee\chi_2\vee\dots\vee\chi_n).
\end{eqnarray*}
Note that $\chi_2\vee\dots\vee\chi_{n}\vdash\chi_1\vee\dots\vee\chi_{n-1}\vee\chi_n$ is provable in the propositional logic. Thus,
\begin{eqnarray}
&&\hspace{-15mm}\{\cK_{D_i}\B_{D_i}\chi_i\}_{i=1}^n,\;\chi_2\vee\dots\vee\chi_n \vdash \B_{D_1\cup\dots\cup D_{n-1}\cup D_n}(\chi_1\vee\chi_2\vee\dots\vee\chi_n).\label{part 2}
\end{eqnarray}
Finally, note that the following statement is provable in the propositional logic for $n\ge 2$,
$$
\vdash\chi_1\vee\dots\vee\chi_n\to(\chi_1\vee\dots\vee\chi_{n-1})\vee 
(\chi_2\vee\dots\vee\chi_n).
$$
Therefore, from statement~(\ref{part 1}) and statement~(\ref{part 2})
$$
\{\cK_{D_i}\B_{D_i}\chi_i\}_{i=1}^n,\;\chi_1\vee\dots\vee\chi_n
\vdash \B_{D_1\cup\dots\cup D_n}(\chi_1\vee \dots\vee\chi_n).
$$
by the laws of propositional reasoning.
\end{proof}

\begin{lemma}\label{super distributivity}
If $\phi_1,\dots,\phi_n\vdash\psi$, then $\K_C\phi_1,\dots,\K_C\phi_n\vdash\K_C\psi$.
\end{lemma}
\begin{proof}
By Lemma~\ref{deduction lemma} applied $n$ times, assumption $\phi_1,\dots,\phi_n\vdash\psi$ implies that
$
\vdash\phi_1\to(\phi_2\to\dots(\phi_n\to\psi)\dots).
$
Thus, by the Necessitation inference rule,
$$
\vdash\K_C(\phi_1\to(\phi_2\to\dots(\phi_n\to\psi)\dots)).
$$
Hence, by the Distributivity axiom and the Modus Ponens rule,
$$
\vdash\K_C\phi_1\to\K_C(\phi_2\to\dots(\phi_n\to\psi)\dots).
$$
Then, again by the Modus Ponens rule,
$$
\K_C\phi_1\vdash\K_C(\phi_2\to\dots(\phi_n\to\psi)\dots).
$$
Therefore, $\K_C\phi_1,\dots,\K_C\phi_n\vdash\K_C\psi$ by applying the previous steps $(n-1)$ more times.
\end{proof}

The following lemma states a well-known principle in epistemic logic. 

\begin{lemma}[Positive Introspection]\label{positive introspection lemma}
$\vdash \K_C\phi\to\K_C\K_C\phi$. 
\end{lemma}
\begin{proof}
Formula $\K_C\neg\K_C\phi\to\neg\K_C\phi$ is an instance of the Truth axiom. Thus, $\vdash \K_C\phi\to\neg\K_C\neg\K_C\phi$ by contraposition. Hence, taking into account the following instance of  the Negative Introspection axiom: $\neg\K_C\neg\K_C\phi\to\K_C\neg\K_C\neg\K_C\phi$,
we have 
\begin{equation}\label{pos intro eq 2}
\vdash \K_C\phi\to\K_C\neg\K_C\neg\K_C\phi.
\end{equation}

At the same time, $\neg\K_C\phi\to\K_C\neg\K_C\phi$ is an instance of the Negative Introspection axiom. Thus, $\vdash \neg\K_C\neg\K_C\phi\to \K_C\phi$ by the law of contrapositive in the propositional logic. Hence, by the Necessitation inference rule, 
$\vdash \K_C(\neg\K_C\neg\K_C\phi\to \K_C\phi)$. Thus, by  the Distributivity axiom and the Modus Ponens inference rule, 
$
  \vdash \K_C\neg\K_C\neg\K_C\phi\to \K_C\K_C\phi.
$
 The latter, together with statement~(\ref{pos intro eq 2}), implies the statement of the lemma by propositional reasoning.
\end{proof}

Our last example rephrases Lemma~\ref{super joint responsibility lemma} into the form which is used in the proof of the completeness.

\begin{lemma}\label{five plus plus}
For any $n\ge 0$ and any disjoint sets $D_1,\dots,D_n\subseteq C$,
$$
\{\cK_{D_i}\B_{D_i}\chi_i\}_{i=1}^n,\;\K_C(\phi\to\chi_1\vee\dots\vee\chi_n)\vdash\K_C(\phi\to\B_C\phi).
$$
\end{lemma}
\begin{proof}
By Lemma~\ref{super joint responsibility lemma},
$$
\{\cK_{D_i}\B_{D_i}\chi_i\}_{i=1}^n,\;\chi_1\vee\dots\vee\chi_n\vdash \B_{D_1\cup\dots\cup D_n}(\chi_1\vee\dots\vee\chi_n).
$$
Hence, by the Monotonicity axiom, 
$$
\{\cK_{D_i}\B_{D_i}\chi_i\}_{i=1}^n,\;\chi_1\vee\dots\vee\chi_n\vdash \B_{C}(\chi_1\vee\dots\vee\chi_n).
$$
Thus, by the Modus Ponens inference rule,
$$
\{\cK_{D_i}\B_{D_i}\chi_i\}_{i=1}^n,\;\phi,\;\phi\to\chi_1\vee\dots\vee\chi_n\vdash \B_C(\chi_1\vee\dots\vee\chi_n).
$$
By the Truth axiom and the Modus Ponens inference rule,
$$
\{\cK_{D_i}\B_{D_i}\chi_i\}_{i=1}^n,\;\phi,\;\K_C(\phi\to\chi_1\vee\dots\vee\chi_n)\vdash \B_C(\chi_1\vee\dots\vee\chi_n).
$$
The following formula is an instance of the Blame for Known Cause axiom $\K_C(\phi\to\chi_1\vee\dots\vee\chi_n)\to(\B_C(\chi_1\vee\dots\vee\chi_n)\to(\phi\to\B_C\phi))$. Hence, by the Modus Ponens inference rule applied twice,
$$
\{\cK_{D_i}\B_{D_i}\chi_i\}_{i=1}^n,\;\phi,\;\K_C(\phi\to\chi_1\vee\dots\vee\chi_n)\vdash\phi\to\B_C\phi.
$$
By the Modus Ponens inference rule,
$$
\{\cK_{D_i}\B_{D_i}\chi_i\}_{i=1}^n,\;\phi,\; \K_C(\phi\to\chi_1\vee\dots\vee\chi_n)\vdash\B_C\phi.
$$
By Lemma~\ref{deduction lemma},
$$
\{\cK_{D_i}\B_{D_i}\chi_i\}_{i=1}^n,\;\K_C(\phi\to\chi_1\vee\dots\vee\chi_n)\vdash\phi\to\B_C\phi.
$$
By Lemma~\ref{super distributivity},
$$
\{\K_C\cK_{D_i}\B_{D_i}\chi_i\}_{i=1}^n,\;\K_C\K_C(\phi\to\chi_1\vee\dots\vee\chi_n)\vdash\K_C(\phi\to\B_C\phi).
$$
By the Monotonicity axiom, the Modus Ponens inference rule, and the assumption $D_1,\dots,D_n\subseteq C$,
$$
\{\K_{D_i}\cK_{D_i}\B_{D_i}\chi_i\}_{i=1}^n,\;\K_C\K_C(\phi\to\chi_1\vee\dots\vee\chi_n)\vdash\K_C(\phi\to\B_C\phi).
$$
By the definition of modality $\cK$, the Negative Introspection axiom, and the Modus Ponens inference rule,
$$
\{\cK_{D_i}\B_{D_i}\chi_i\}_{i=1}^n,\;\K_C\K_C(\phi\to\chi_1\vee\dots\vee\chi_n)\vdash\K_C(\phi\to\B_C\phi).
$$
Therefore, by Lemma~\ref{positive introspection lemma} and the Modus Ponens inference rule, the statement of the lemma follows. 
\end{proof}

\section{Soundness}\label{soundness section}

The epistemic part of the Truth axiom as well as the Distribitivity, the Negative Introspection, and the Monotonicity axioms are the standard axioms of epistemic logic S5 for distributed knowledge. Their soundness follows from the assumption that $\sim_a$ is an equivalence relation in the standard way~\cite{fhmv95}. The soundness of the blameworthiness part of the Truth axiom  and of the Monotonicity axiom immediately follows from Definition~\ref{sat}. In this section, we prove the soundness of each of the remaining axioms as a separate lemma.
In these lemmas, $C,D\subseteq\mathcal{A}$ are coalitions, $\phi,\psi\in \Phi$ are formulae, and $(\alpha,\delta,\omega)\in P$ is a play of a  game $(I,\{\sim_a\}_{a\in\mathcal{A}},\Delta,\Omega,P,\pi)$.  




\begin{lemma}
$(\alpha,\delta,\omega)\nVdash \B_\varnothing\phi$. 
\end{lemma}
\begin{proof}
Assume that $(\alpha,\delta,\omega)\Vdash \B_\varnothing\phi$. Hence, by Definition~\ref{sat}, we have $(\alpha,\delta,\omega)\Vdash \phi$ and there is an action profile $s\in\Delta^\varnothing$ such that for each play $(\alpha',\delta',\omega')\in P$, if $\alpha\sim_\varnothing\alpha'$ and $s=_\varnothing\delta'$, then $(\alpha',\delta',\omega')\nVdash\phi$.

Let $\alpha'=\alpha$, $\delta'=\delta$, and $\omega'=\omega$. Since $\alpha\sim_\varnothing\alpha'$ and $s=_\varnothing\delta'$, by the choice of action profile $s$ we have $(\alpha',\delta',\omega')\nVdash\phi$. Then, $(\alpha,\delta,\omega)\nVdash\phi$, which leads to a contradiction.
\end{proof}

\begin{lemma}
$(\alpha,\delta,\omega)\nVdash \B_C\top$. 
\end{lemma}
\begin{proof}
Suppose that $(\alpha,\delta,\omega)\Vdash \B_C\top$. Thus, by Definition~\ref{sat}, there  is an action profile $s\in \Delta^C$ of coalition $C$ such that for each play $(\alpha',\delta',\omega')\in P$, if $\alpha\sim_C\alpha'$ and $s=_C\delta'$, then $(\alpha',\delta',\omega')\nVdash\top$.

Recall that the set of actions $\Delta$ is not empty by Definition~\ref{game definition}. Let $d_0$ be any action from set $\Delta$. Define a complete action profile $\delta'\in \Delta^\mathcal{A}$ as follows:
$$
\delta'(a)=
\begin{cases}
s(a), & \mbox{ if } a\in C,\\
d_0, & \mbox{ otherwise.}
\end{cases}
$$
By item 5 of Definition~\ref{game definition}, there is an outcome $\omega'\in \Omega$ such that $(\alpha,\delta',\omega')\in P$. Note that  $\alpha\sim_C\alpha$ because relation $\sim_C$ is an equivalence relation. Also $s=_C\delta'$ by the choice of the complete action profile $\delta'$. Therefore, by the choice of the action profile $s\in \Delta^C$, we have $(\alpha,\delta',\omega')\nVdash\top$, which contradicts Definition~\ref{sat}, taking into account the definition of the constant $\top$. 
\end{proof}

\begin{lemma}
If $C\cap D=\varnothing$, $(\alpha,\delta,\omega)\Vdash \cK_C\B_C\phi$, $(\alpha,\delta,\omega)\Vdash \cK_D\B_D\psi$, and $(\alpha,\delta,\omega)\Vdash \phi\vee\psi$, then $(\alpha,\delta,\omega)\Vdash \B_{C\cup D}(\phi\vee\psi)$.
\end{lemma}
\begin{proof}
Suppose that $(\alpha,\delta,\omega)\Vdash \cK_C\B_C\phi$ and $(\alpha,\delta,\omega)\Vdash \cK_D\B_D\psi$. Hence, by Definition~\ref{sat} and the definition of modality $\cK$, there are plays $(\alpha_1,\delta_1,\omega_1)\in P$ and $(\alpha_2,\delta_2,\omega_2)\in P$ such that $\alpha\sim_C\alpha_1$, $\alpha\sim_D\alpha_2$, $(\alpha_1,\delta_1,\omega_1)\Vdash \B_C\phi$ and $(\alpha_2,\delta_2,\omega_2)\Vdash \B_D\psi$.

Statement $(\alpha_1,\delta_1,\omega_1)\Vdash \B_C\phi$, by Definition~\ref{sat}, implies that there is a profile $s_1\in \Delta^C$ such that for each play $(\alpha',\delta',\omega')\in P$, if $\alpha_1\sim_C\alpha'$ and $s_1=_C\delta'$, then $(\alpha',\delta',\omega')\nVdash\phi$.

Similarly, statement $(\alpha_2,\delta_2,\omega_2)\Vdash \B_D\psi$, by Definition~\ref{sat}, implies that there is an action profile $s_2\in \Delta^D$ such that for each play $(\alpha',\delta',\omega')\in P$, if $\alpha_2\sim_D\alpha'$ and $s_2=_D\delta'$, then $(\alpha',\delta',\omega')\nVdash\psi$.

Consider an action profile $s$ of coalition $C\cup D$ such that
$$
s(a)=
\begin{cases}
s_1(a), & \mbox{ if } a\in C,\\
s_2(a), & \mbox{ if } a\in D.
\end{cases}
$$
The action profile $s$ is well-defined because sets $C$ and $D$ are disjoint by the assumption of the lemma. 

The choice of action profiles $s_1$, $s_2$, and $s$ implies that  for each play $(\alpha',\delta',\omega')\in P$, if $\alpha\sim_{C\cup D}\alpha'$ and $s=_{C\cup D}\delta'$, then $(\alpha',\delta',\omega')\nVdash\phi$ and $(\alpha',\delta',\omega')\nVdash\psi$. 
Thus, if $\alpha\sim_{C\cup D}\alpha'$ and $s=_{C\cup D}\delta'$, then $(\alpha',\delta',\omega')\nVdash\phi\vee\psi$, for each play $(\alpha',\delta',\omega')\in P$. 
Therefore, $(\alpha,\delta,\omega)\Vdash \B_{C\cup D}(\phi\vee\psi)$  by Definition~\ref{sat} and the assumption $(\alpha,\delta,\omega)\Vdash \phi\vee\psi$ of the lemma.
\end{proof}

\begin{lemma}
If $(\alpha,\delta,\omega)\Vdash \K_C(\phi\to\psi)$, $(\alpha,\delta,\omega)\Vdash \B_C\psi$, and $(\alpha,\delta,\omega)\Vdash \phi$, then $(\alpha,\delta,\omega)\Vdash \B_C\phi$.
\end{lemma}

\begin{proof}
By Definition~\ref{sat}, assumption $(\alpha,\delta,\omega)\Vdash \K_C(\phi\to\psi)$ implies that for each play $(\alpha',\delta',\omega')\in P$ of the game if $\alpha\sim_C\alpha'$, then
$(\alpha',\delta',\omega')\Vdash\phi\to\psi$. 

By Definition~\ref{sat}, assumption $(\alpha,\delta,\omega)\Vdash \B_C\psi$ implies that there is an action profile $s\in \Delta^C$ such that for each play $(\alpha',\delta',\omega')\in P$, if $\alpha\sim_C\alpha'$ and $s=_C\delta'$, then $(\alpha',\delta',\omega')\nVdash\psi$. 

Hence, for each play $(\alpha',\delta',\omega')\in P$, if $\alpha\sim_C\alpha'$ and $s=_C\delta'$, then $(\alpha',\delta',\omega')\nVdash\phi$.
Therefore, $(\alpha,\delta,\omega)\Vdash \B_C\phi$ by Definition~\ref{sat} and the assumption $(\alpha,\delta,\omega)\Vdash \phi$ of the lemma.
\end{proof}

\begin{lemma}
If $(\alpha,\delta,\omega)\Vdash \B_C\phi$, then $(\alpha,\delta,\omega)\Vdash \K_C(\phi\to\B_C\phi)$.
\end{lemma}
\begin{proof} 
By Definition~\ref{sat}, assumption $(\alpha,\delta,\omega)\Vdash \B_C\phi$ implies that there is an action profile $s\in \Delta^C$ such that for each play $(\alpha',\delta',\omega')\in P$, if $\alpha\sim_C\alpha'$ and $s=_C\delta'$, then $(\alpha',\delta',\omega')\nVdash\phi$. 

Let $(\alpha',\delta',\omega')\in P$ be a play where $\alpha\sim_C\alpha'$ and $(\alpha',\delta',\omega')\Vdash \phi$. By Definition~\ref{sat}, it suffices to show that $(\alpha',\delta',\omega')\Vdash \B_C\phi$. 

Consider any play $(\alpha'',\delta'',\omega'')\in P$ such that $\alpha'\sim_C\alpha''$ and $s=_C\delta''$.  
Then, since $\sim_C$ is an equivalence relation, assumptions $\alpha\sim_C\alpha'$ and $\alpha'\sim_C\alpha''$ imply $\alpha\sim_C\alpha''$. Thus, $(\alpha'',\delta'',\omega'')\nVdash\phi$ by the choice of action profile $s$. Therefore, $(\alpha',\delta',\omega')\Vdash \B_C\phi$ by Definition~\ref{sat} and the assumption $(\alpha',\delta',\omega')\Vdash \phi$.
\end{proof}

\section{Completeness}\label{completeness section}

In this section we prove the completeness of our logical system. The completeness theorem is stated in the end of this section as Theorem~\ref{completeness theorem}.

The standard completeness proof for epistemic logic of individual knowledge defines states as maximal consistent sets. Similarly, we defined outcomes of the game as maximal consistent sets in~\cite{nt19aaai}. In the case of the epistemic logic of distributed knowledge, two states are usually defined to be indistinguishable by an agent $a$ if these two states have the same $\K_a$ formulae. Unfortunately, this approach does not work for distributed knowledge. Indeed, two maximal consistent sets that have the same $\K_a$ and $\K_b$ formulae might have different $\K_{a,b}$ formulae. Such two states would be indistinguishable to agent $a$ and agent $b$, however, the distributed knowledge of agents $a$ and $b$ in these states will be different. This situation is inconsistent with Definition~\ref{sat}. To solve this problem we define outcomes not as maximal consistent sets of formulae, but as nodes of a tree. This approach has been previously used to prove the completeness of several logics for know-how modality~\cite{nt17aamas,nt17tark,nt18ai,nt18aaai,nt18aamas}.

We start the proof of
the completeness by defining the canonical game $G(X_0)=\left(I,\{\sim_a\}_{a\in\mathcal{A}},\Delta,\Omega,P,\pi\right)$ for each maximal consistent set of formulae $X_0$. In this definition, $\Phi$ refers to the set of all formulae in our language, see Definition~\ref{Phi}.

\begin{definition}\label{canonical outcome}
The set of outcomes $\Omega$ consists of all finite sequences $X_0,C_1,X_1$, $C_2,\dots,C_n,X_n$, such that
\begin{enumerate}
    \item $n\ge 0$,
    \item $X_i$ is a maximal consistent subset of $\Phi$ for each $i\ge 1$,
    \item $C_i$ is a coalition for each $i\ge 1$,
    \item $\{\phi\;|\;\K_{C_i}\phi\in X_{i-1}\}\subseteq X_i$ for each $i\ge 1$.
\end{enumerate}
\end{definition}

For any sequence $s=x_1,\dots,x_n$ and any element $y$, by $s::y$ we mean the sequence $x_1,\dots,x_n,y$. By $hd(s)$ we mean element $x_n$.
\begin{figure}[ht]
\begin{center}
\vspace{-1mm}
\scalebox{0.6}{\includegraphics{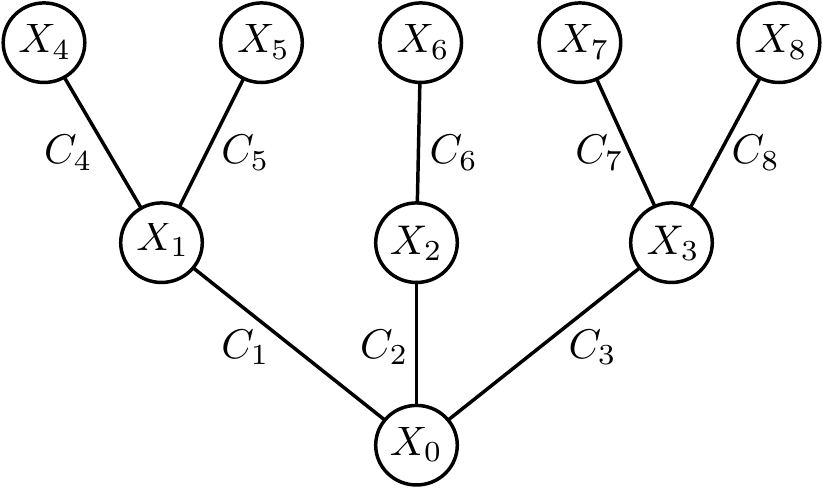}}
\caption{A fragment of tree.}\label{tree figure}
\vspace{-7mm}
\end{center}
\end{figure}
We define a tree structure on the set of outcomes $\Omega$ by saying that outcome (node) $\omega=X_0,C_1,X_1,C_2,\dots,C_n,X_n$ and outcome (node) $\omega::C_{n+1}::X_{n+1}$ are connected by an undirected edge labeled with all agents in coalition $C_{n+1}$, see Figure~\ref{tree figure}.

\begin{definition}\label{canonical sim}
For any outcomes $\omega,\omega'\in\Omega$ and any agent $a\in\mathcal{A}$, let $\omega\sim_a\omega'$ if all edges along the unique path between $\omega$ and $\omega'$ are labeled with agent $a$.
\end{definition}
\begin{lemma}\label{canonical sim is equivalence relation}
Relation $\sim_a$ is an equivalence relation on set $\Omega$.
\qed
\end{lemma}

Lemma~\ref{pre-transport lemma} below shows that the tree construction overcomes the distributed knowledge challenge discussed in the preamble for this section. Lemma~\ref{edge transport lemma} lays ground for the induction step in the proof of Lemma~\ref{pre-transport lemma}.

\begin{lemma}\label{edge transport lemma}
$\K_D\phi\in X_n$ iff $\K_D\phi\in X_{n+1}$ for any formula $\phi\in\Phi$, any $n\ge 0$, and any outcome $X_0,C_1,X_1,C_2,\dots,X_n,C_{n+1},X_{n+1}\in\Omega$, and any coalition $D\subseteq C_{n+1}$.
\end{lemma}
\begin{proof}
If $\K_D\phi\in X_n$, then $X_n\vdash\K_D\K_D\phi$ by Lemma~\ref{positive introspection lemma}. Hence, $X_n\vdash\K_{C_{n+1}}\K_D\phi$ by the Monotonicity axiom and the assumption $D\subseteq {C_{n+1}}$. Thus, $\K_{C_{n+1}}\K_D\phi\in X_n$ by the maximality of set $X_n$. Therefore, $\K_D\phi\in X_{n+1}$ by Definition~\ref{canonical outcome}.

Suppose that $\K_D\phi\notin X_n$. Hence, $\neg\K_D\phi\in X_n$ by the maximality of set $X_n$. Thus, $X_n\vdash \K_D\neg\K_D\phi$ by the Negative Introspection axiom. Hence,  $X_n\vdash \K_{C_{n+1}}\neg\K_D\phi$ by the Monotonicity axiom and the assumption $D\subseteq C_{n+1}$. Then, $\K_{C_{n+1}}\neg\K_D\phi\in X_n$ by the maximality of set $X_n$. Thus, $\neg\K_D\phi\in X_{n+1}$ by Definition~\ref{canonical outcome}. Therefore, $\K_D\phi\notin X_{n+1}$ because set $X_{n+1}$ is consistent. 
\end{proof}

\begin{lemma}\label{pre-transport lemma}
If $\omega\sim_C \omega'$, then $\K_C\phi\in hd(\omega)$ iff $\K_C\phi\in hd(\omega')$.
\end{lemma}
\begin{proof}
If $\omega\sim_C \omega'$, then each edge along the unique path between nodes $\omega$ and $\omega'$ is labeled with all agents in coalition $C$. 

We prove the lemma by induction on the length of the unique path between nodes $\omega$ and $\omega'$. In the base case, $\omega=\omega'$. Thus, $\K_C\phi\in hd(\omega)$ iff $\K_C\phi\in hd(\omega')$. The induction step follows from Lemma~\ref{edge transport lemma}.
\end{proof}

\begin{lemma}\label{transport lemma}
If $\omega\sim_C \omega'$ and $\K_C\phi\in hd(\omega)$, then $\phi\in hd(\omega')$.
\end{lemma}
\begin{proof}
By Lemma~\ref{pre-transport lemma}, assumptions $\omega\sim_C \omega'$ and $\K_C\phi\in hd(\omega)$ imply that $\K_C\phi\in hd(\omega')$. Thus, $hd(\omega')\vdash\phi$ by the Truth axiom and the Modus Ponens inference rule. Therefore, $\phi\in hd(\omega')$ because set $hd(\omega')$ is maximal.
\end{proof}

The set of the initial states $I$ of the canonical game is the set of all equivalence classes of $\Omega$ with respect to relation $\sim_\mathcal{A}$.

\begin{definition}
$I=\Omega/\!\sim_\mathcal{A}$.
\end{definition}

\begin{lemma}\label{well-defined lemma}
Relation $\sim_C$ is well-defined on set $I$.
\end{lemma}
\begin{proof}
Suppose that $\omega_1\sim_C\omega_2$. Consider any outcomes $\omega'_1$ and $\omega'_2$ such that $\omega_1\sim_\mathcal{A}\omega'_1$ and $\omega_2\sim_\mathcal{A}\omega'_2$. It suffices to prove that $\omega'_1\sim_C\omega'_2$.

By Definition~\ref{canonical sim} and Lemma~\ref{canonical sim is equivalence relation}, assumption $\omega_1\sim_\mathcal{A}\omega'_1$ implies that each edges along the unique path between nodes $\omega'_1$ and $\omega_1$ is labeled with all agents in set $\mathcal{A}$. Also,  assumption $\omega_1\sim_C\omega_2$ implies that each edge along the unique path between nodes $\omega_1$ and $\omega_2$ is labeled with all agents in coalition $C$. Finally, assumption $\omega_2\sim_\mathcal{A}\omega'_2$ implies that each edges along the unique path between nodes $\omega_2$ and $\omega'_2$ is labeled with all agents in set $\mathcal{A}$. Hence, each edge along the unique path between nodes $\omega'_1$ and $\omega'_2$ is labeled with all agents in coalition $C$. Therefore, $\omega'_1\sim_C\omega'_2$ by Definition~\ref{canonical sim}.
\end{proof}

\begin{lemma}\label{alpha iff omega}
$\alpha\sim_C\alpha'$ iff $\omega\sim_C\omega'$, for any initial states $\alpha,\alpha'\in I$, any outcomes $\omega\in\alpha$ and $\omega'\in\alpha'$, and any coalition $C\subseteq\mathcal{A}$. \qed
\end{lemma}
If $C=\mathcal{A}$, then the above lemma implies that elements of class $\alpha$ are $\sim_{\mathcal{A}}$ equivalent to elements of class $\alpha'$. Thus, we have the following corollary.

\begin{corollary}\label{alpha corollary}
For any initial states $\alpha, \alpha'\in I$, if $\alpha\sim_\mathcal{A}\alpha'$, then $\alpha=\alpha'$. 
\qed
\end{corollary}

Intuitively, in the canonical game, the agents ``veto'' formulae. The domain of choices of the game consists of all formulae in the set $\Phi$. To veto a formula $\psi$, an agent must choose action $\psi$.  The mechanism guarantees that if $\cK_C\B_C\psi\in hd(\omega)$ and all agents in the coalition $C$ veto formula $\psi$, then $\neg\psi\in hd(\omega)$.

\begin{definition}
The domain of actions $\Delta$ is set $\Phi$.
\end{definition}

\begin{definition}\label{canonical play}
The set $P\subseteq I\times \Delta^\mathcal{A}\times \Omega$ consists of all triples $(\alpha,\delta,\omega)$ such that $\omega\in\alpha$ and for  any formula $\cK_C\B_C\psi\in hd(\omega)$, if $\delta(a)=\psi$ for each agent $a\in C$, then $\neg\psi\in hd(\omega)$.
\end{definition}

\begin{definition}\label{canonical pi}
$\pi(p)=\{(\alpha,\delta,\omega)\in P\;|\; p\in hd(\omega)\}$.
\end{definition}

This concludes the definition of the canonical game $G(X_0)$. In Lemma~\ref{canonical nontermination lemma} we will show the condition from item 5 of Definition~\ref{game definition}. Namely, that for each initial state $\alpha\in I$ and each complete action profile $\delta\in\Delta^\mathcal{A}$ there is at least one outcome $\omega\in \Omega$ such that $(\alpha,\delta,\omega)\in P$.  

We state and prove the completeness later in this section as Theorem~\ref{completeness theorem}. We start with auxiliary results that will be used in the proof of the completeness.

\begin{lemma}\label{B child exists lemma}
For any play $(\alpha,\delta,\omega)\in P$ of  game $G(X_0)$, any action profile $s\in\Delta^C$, and any formula $\neg(\phi\to \B_C\phi)\in hd(\omega)$, there is a play $(\alpha',\delta',\omega')\in P$ such that $\alpha\sim_C\alpha'$, $s =_C\delta'$, and $\phi\in hd(\omega')$.
\end{lemma}
\begin{proof}
Consider the following set of formulae:
\begin{eqnarray*}
X&\!\!=\!\!&\!\{\phi\}\;\cup\;\{\psi\;|\;\K_C\psi\in hd(\omega)\}\\
&&\!\cup\;\{\neg\chi\;|\;\cK_D\B_D\chi\in hd(\omega), D\subseteq C,\forall a\in D(s(a)=\chi)\}.
\end{eqnarray*}
\begin{claim}
Set $X$ is consistent.
\end{claim}
\begin{proof-of-claim}
Suppose the opposite. Thus, there are 
\begin{eqnarray}
\mbox{ formulae }&&\K_C\psi_1,\dots,\K_C\psi_m\in hd(\omega),\label{choice of psi-s}\\
\mbox{and formulae }&&\cK_D\B_{D_1}\chi_1,\dots,\cK_D\B_{D_n}\chi_n\in hd(\omega),\label{choice of chi-s}\\
\mbox{such that }&&D_1,\dots,D_n\subseteq C,\label{choice of Ds}\\
&&s(a)=\chi_i\mbox{ for all } i\le n\mbox{ and all }a\in D_i,\label{choice of votes}\\
\mbox{ and }&&\psi_1,\dots,\psi_m,\neg\chi_1,\dots,\neg\chi_n\vdash\neg\phi.\label{choice of cons}
\end{eqnarray}
Without loss of generality, we assume that formulae $\chi_1,\dots,\chi_n$ are distinct. Thus, assumption~(\ref{choice of votes}) implies that sets $D_1,\dots,D_n$ are pairwise disjoint. 
By propositional reasoning, assumption~(\ref{choice of cons}) implies 
$$
\psi_1,\dots,\psi_m\vdash\phi\to\chi_1\vee\dots\vee\chi_n.
$$
Thus, by Lemma~\ref{super distributivity},
$$
\K_C\psi_1,\dots,\K_C\psi_m\vdash\K_C(\phi\to\chi_1\vee\dots\vee\chi_n).
$$
Hence, 
$
hd(\omega)\vdash\K_C(\phi\to\chi_1\vee\dots\vee\chi_n)
$ 
by assumption~(\ref{choice of psi-s}).
Thus, $
hd(\omega)\vdash \K_C(\phi\to\B_C\phi)
$ by  Lemma~\ref{five plus plus}, assumption~(\ref{choice of chi-s}), and the assumption that sets $D_1,\dots,D_n$ are pairwise disjoint.
Hence, by the Truth axiom, $hd(\omega)\vdash \phi\to\B_C\phi$, which contradicts the assumption $\neg(\phi\to\B_C\phi)\in hd(\omega)$ of the lemma because set $hd(\omega)$ is consistent.
Therefore, set $X$ is consistent.
\end{proof-of-claim}

By Lemma~\ref{Lindenbaum's lemma}, there is a maximal consistent extension $X'$ of set $X$. Let $\omega'$ be the sequence $\omega::C::X'$. Note that $\omega'\in\Omega$ by Definition~\ref{canonical outcome} and the choice of sets $X$ and $X'$. Also $\phi\in X\subseteq hd(\omega')$ by the choice of sets $X$ and $X'$.

Let initial state $\alpha'$ be the equivalence class of outcome $\omega'$ with respect to the equivalence relation $\sim_{\mathcal{A}}$. Note that $\omega\sim_C\omega'$ by Definition~\ref{canonical outcome} and the choice of sequence $\omega'$. Therefore, $\alpha\sim_C\alpha'$ by Lemma~\ref{alpha iff omega}.

Let the complete action profile $\delta'$ be defined as follows:
\begin{equation}\label{choice of delta'}
    \delta'(a)=
    \begin{cases}
    s(a), & \mbox{ if } a\in C,\\
    \bot, & \mbox{ otherwise}.
    \end{cases}
\end{equation}
Then, $s=_C\delta'$.
\begin{claim}
$(\alpha',\delta',\omega')\in P$.
\end{claim}
\begin{proof-of-claim}
First, note that $\omega'\in\alpha'$ because state $\alpha'$ is the equivalence class of outcome $\omega'$. Next, consider any formula $\cK_D\B_D\chi\in hd(\omega')$ such that $\delta'(a)=\chi$ for each $a\in D$. By Definition~\ref{canonical play}, it suffices to show that $\neg\chi\in hd(\omega')$. 

\noindent{\bf Case I:} $D\subseteq C$. Thus, $s(a)=\chi$ for each $a\in D$ by equation~(\ref{choice of delta'}) and the assumption that $\delta'(a)=\chi$ for each $a\in D$.

Suppose that $\neg\chi\notin hd(\omega')$. Then, $\neg\chi\notin X$ because $X\subseteq X'=hd(\omega')$ by the choice of  $X'$ and  $\omega'$. Thus, $\cK_D\B_D\chi\notin hd(\omega)$ by the definition of set $X$ and because  $s(a)=\chi$ for each $a\in D$. Hence, $\K_D\neg\B_D\chi\in hd(\omega)$ by the definition of modality $\cK$ and the maximality of the set $hd(\omega)$. Thus, $hd(\omega)\vdash \K_D\K_D\neg\B_D\chi$ by Lemma~\ref{positive introspection lemma}.  Then, $hd(\omega)\vdash \K_C\K_D\neg\B_D\chi$ by the Monotonicity axiom and because $D\subseteq C$. Thus, $\K_C\K_D\neg\B_D\chi\in hd(\omega)$ by the maximality of the set $hd(\omega)$. Hence,  $\K_D\neg\B_D\chi\in X$ by the choice of set $X$. Thus,  $\K_D\neg\B_D\chi\in X'=hd(\omega')$ by the choice of set $X'$ and the choice of sequence $\omega'$. Then, $\neg\K_D\neg\B_D\chi\notin hd(\omega')$ because set $hd(\omega')$ is consistent. Therefore, $\cK_D\B_D\chi\notin hd(\omega')$  by the definition of modality $\cK$, which contradicts the choice of formula $\cK_D\B_D\chi$.

\noindent{\bf Case II:} $D\nsubseteq C$. Consider any $d_0\in D\setminus C$. Thus, $\delta'(d_0)=\bot$ by equation~(\ref{choice of delta'}). Also, $\delta'(d_0)=\chi$ because $d_0\in D$. Thus, $\chi\equiv \bot$. Hence, formula $\neg\chi$ is a tautology. Therefore, $\neg\chi\in hd(\omega')$ by the maximality of set $hd(\omega')$. 
\end{proof-of-claim}

\noindent This concludes the proof of the lemma.
\end{proof}

\begin{lemma}\label{delta exists lemma}
For any outcome $\omega\in\Omega$, there is an initial state $\alpha\in I$ and a complete action profile $\delta\in \Delta^\mathcal{A}$ such that $(\alpha,\delta,\omega)\in P$.
\end{lemma}
\begin{proof}
Let the initial state $\alpha$ be the equivalence class of outcome $\omega$ with respect to the equivalence relation $\sim_{\mathcal{A}}$. Thus, $\omega\in\alpha$. Let $\delta$ be the complete action profile such that $\delta(a)=\bot$ for each $a\in \mathcal{A}$. To prove $(\alpha,\delta,\omega)\in P$, consider any formula $\cK_D\B_D\chi\in hd(\omega)$ 
such that $\delta(a)=\chi$ for each $a\in D$. By Definition~\ref{canonical play}, it suffices to show that $\neg\chi\in hd(\omega)$. 

\noindent{\bf Case I}: $D=\varnothing$. Thus, $\vdash\neg\B_D\chi$ by the None to Blame axiom. Hence, $\vdash\K_D\neg\B_D\chi$ by the Necessitation rule. Then, $\neg\K_D\neg\B_D\chi\notin hd(\omega)$ because set $hd(\omega)$ is consistent. Therefore, $\cK_D\B_D\chi\notin hd(\omega)$ by the definition of modality $\cK$, which contradicts the choice of formula $\cK_D\B_D\chi$. 

\noindent{\bf Case II}: $D\neq\varnothing$. Then, there is at least one agent $d_0\in D$. Hence, $\chi=\delta(d_0)=\bot$ by the definition of the complete action profile $\delta$. Then, $\neg\chi$ is a tautology. Thus, $\neg\chi\in hd(\omega)$ by the maximality of set $hd(\omega)$.
\end{proof}

Next we show that the canonical model satisfies the condition from item 5 of Definition~\ref{game definition}.

\begin{lemma}\label{canonical nontermination lemma}
For each initial state $\alpha\in I$ and each complete action profile $\delta\in \Delta^\mathcal{A}$, there is an outcome $\omega\in\Omega$ such that $(\alpha,\delta,\omega)\in P$. 
\end{lemma}
\begin{proof}
By Definition~\ref{canonical outcome}, initial state $\alpha$ is an equivalence class. Since each equivalence class is not empty, there must exist an outcome $\omega_0\in \Omega$ such that $\omega_0\in \alpha$. By Lemma~\ref{delta exists lemma}, there is an initial state $\alpha_0\in I$ and a complete action profile $\delta_0\in \Delta^\mathcal{A}$ such that $(\alpha_0,\delta_0,\omega_0)\in P$. Then, $\omega_0\in \alpha_0$ by Definition~\ref{canonical play}. Hence, $\omega_0$ belongs to equivalence classes $\alpha$ and $\alpha_0$. Thus, $\alpha=\alpha_0$. Therefore, $(\alpha,\delta_0,\omega_0)\in P$.

Note that $\neg\B_{\mathcal{A}}\top$ is an instance of the Blamelessness of Truth axiom. Thus, by the laws of propositional reasoning, $\vdash \neg(\top\to \B_{\mathcal{A}}\top)$. Hence, $\neg(\top\to \B_{\mathcal{A}}\top)\in hd(\omega_0)$ because set $hd(\omega_0)$ is maximal. Thus, by Lemma~\ref{B child exists lemma}, applied to play $(\alpha,\delta_0,\omega_0)\in P$, the action profile $\delta\in\Delta^\mathcal{A}$, and formula $\neg(\top\to \B_{\mathcal{A}}\top)\in hd(\omega_0)$, there is a play $(\alpha',\delta',\omega)\in P$ such that $\alpha\sim_{\mathcal{A}}\alpha'$, $\delta=_{\mathcal{A}}\delta'$, and $\top\in hd(\omega)$. Then, $\alpha=\alpha'$ by Corollary~\ref{alpha corollary}.
Therefore,
$(\alpha,\delta,\omega)=(\alpha',\delta',\omega)\in P$.
\end{proof}

\begin{lemma}\label{N child exists lemma}
For any $(\alpha,\delta,\omega)\in P$ and any  $\neg\K_C\phi\in hd(\omega)$,  there is a play $(\alpha',\delta',\omega')\in P$ such that $\alpha\sim_C\alpha'$ and $\neg\phi\in hd(\omega')$.
\end{lemma}
\begin{proof}
Consider the set $X=\{\neg\phi\}\;\cup\;\{\psi\;|\;\K_C\psi\in hd(\omega)\}$. First, we show that set $X$ is consistent. Suppose the opposite. Then, there are formulae  $\K_C\psi_1,\dots,\K_C\psi_n\in hd(\omega)$
such that
$
\psi_1,\dots,\psi_n\vdash\phi.
$
Hence, 
$
\K_C\psi_1,\dots,\K_C\psi_n\vdash\K_C\phi
$
by Lemma~\ref{super distributivity}.
Thus, $hd(\omega)\vdash\K_C\phi$ because $\K_C\psi_1,\dots,\K_C\psi_n\in hd(\omega)$. Hence, $\neg\K_C\phi\notin hd(\omega)$ because set $hd(\omega)$ is consistent, which contradicts the assumption of the lemma. Therefore, set $X$ is consistent.

By Lemma~\ref{Lindenbaum's lemma}, there is a maximal consistent extension $X'$ of set $X$. Let $\omega'$ be the sequence $\omega::C::X'$. Note that $\omega'\in\Omega$ by Definition~\ref{canonical outcome} and the choice of sets $X$ and $X'$.  Also, $\neg\phi\in X\subseteq X'=hd(\omega')$ by the choice of sets $X$ and $X'$. 

By Lemma~\ref{delta exists lemma}, there is an initial state $\alpha'\in I$ and a complete action profile $\delta'$ such that $(\alpha',\delta',\omega')\in P$.  Note that $\omega\sim_C\omega'$ by Definition~\ref{canonical sim} and the choice of sequence $\omega'$. Thus, $\alpha\sim_C\alpha'$ by Lemma~\ref{alpha iff omega}.
\end{proof}

The next lemma is the ``induction'' lemma,  also known as the ``truth'' lemma, that connects the syntax of our logical system with the semantics of the canonical model.

\begin{lemma}\label{induction lemma}
$(\alpha,\delta,\omega)\Vdash\phi$ iff $\phi\in hd(\omega)$ for each play $(\alpha,\delta,\omega)\in P$ and each formula $\phi\in\Phi$.
\end{lemma}
\begin{proof}
We prove the lemma by induction on the complexity of formula $\phi$. If $\phi$ is a propositional variable, then the lemma follows from Definition~\ref{sat} and Definition~\ref{canonical pi}. If formula $\phi$ is an implication or a negation, then the required follows from the maximality and the consistency of set $\omega$ by Definition~\ref{sat} in the standard way.

Assume that formula $\phi$ has the form $\K_C\psi$.

\noindent $(\Rightarrow):$ Let $\K_C\psi\notin hd(\omega)$. Thus, $\neg\K_C\psi\in hd(\omega)$ by the maximality of set $hd(\omega)$. Hence, by Lemma~\ref{N child exists lemma}, there is a play $(\alpha',\delta',\omega')\in P$ such that $\alpha\sim_C\alpha'$ and $\neg\psi\in hd(\omega')$. Then, $\psi\notin hd(\omega')$ by the consistency of set $hd(\omega')$. Thus, $(\alpha',\delta',\omega')\nVdash\psi$ by the induction hypothesis. Therefore, $(\alpha,\delta,\omega)\nVdash\K_C\psi$ by Definition~\ref{sat}.

\vspace{.5mm}

\noindent $(\Leftarrow):$ Let $\K_C\psi\in hd(\omega)$. Thus, $\psi\in hd(\omega')$ for any $\omega'\in\Omega$ such that $\omega\sim_C\omega'$, by Lemma~\ref{transport lemma}. Hence, by the induction hypothesis, $(\alpha',\delta',\omega')\Vdash\psi$ for each play $(\alpha',\delta',\omega')\in P$ such that  $\omega\sim_C\omega'$. Thus, $(\alpha',\delta',\omega')\Vdash\psi$ for each  $(\alpha',\delta',\omega')\in P$ such that  $\alpha\sim_C\alpha'$, by Lemma~\ref{alpha iff omega}. Therefore, $(\alpha,\delta,\omega)\Vdash\K_C\psi$ by Definition~\ref{sat}. 

Assume that formula $\phi$ has the form $\B_C\psi$. 

\noindent $(\Rightarrow):$ Suppose $\B_C\psi\notin hd(\omega)$. First, consider the case when $\psi\notin hd(\omega)$. Then, $(\alpha,\delta,\omega)\nVdash\psi$ by the induction hypothesis. Thus, $(\alpha,\delta,\omega)\nVdash\B_C\psi$ by Definition~\ref{sat}. 

Next, suppose  $\psi\in hd(\omega)$. Observe that $\psi\to\B_C\psi\notin hd(\omega)$. Indeed, if $\psi\to\B_C\psi\in hd(\omega)$, then $hd(\omega)\vdash \B_C\psi$ by the Modus Ponens inference rule. Thus, $\B_C\psi\in hd(\omega)$ by the  maximality of set $hd(\omega)$, which contradicts the assumption above.

Because $hd(\omega)$ is a maximal set, statement $\psi\to\B_C\psi\notin hd(\omega)$ implies that $\neg(\psi\to\B_C\psi)\in hd(\omega)$. Hence, by Lemma~\ref{B child exists lemma}, for any action profile $s\in \Delta^C$, there is a play $(\alpha',\delta',\omega')$ such that $\alpha\sim_C\alpha'$ and $\psi\in hd(\omega')$. Thus, by the induction hypothesis, for any action profile $s\in \Delta^C$, there is a play $(\alpha',\delta',\omega')$ such that $\alpha\sim_C\alpha'$ and $(\alpha',\delta',\omega')\Vdash \psi$. Therefore, $(\alpha,\delta,\omega)\nVdash\B_C\psi$ by Definition~\ref{sat}.

\vspace{1mm}
\noindent $(\Leftarrow):$ Let $\B_C\psi\in hd(\omega)$. Hence, $hd(\omega)\vdash\psi$ by the Truth axiom. Thus, $\psi\in hd(\omega)$ by the maximality of the set $hd(\omega)$. Then, $(\alpha,\delta,\omega)\Vdash\psi$ by the induction hypothesis.

Next, let $s\in \Delta^C$ be the action profile of coalition $C$ such that $s(a)=\psi$ for each agent $a\in C$. Consider any play $(\alpha',\delta',\omega')\in P$ such that $\alpha\sim_C\alpha'$ and $s=_C\delta'$. By Definition~\ref{sat}, it suffices to show that  $(\alpha',\delta',\omega')\nVdash \psi$. 

Indeed, by Lemma~\ref{add cK lemma}, assumption $\B_C\psi\in hd(\omega)$  implies that $hd(\omega)\vdash \cK_C\B_C\psi$. Thus, $hd(\omega)\vdash \K_C\cK_C\B_C\psi$ by the Negative introspection axiom, the Modus Ponens inference rule, and the definition of modality $\cK$. Hence, $\K_C\cK_C\B_C\psi\in hd(\omega)$ by the maximality of set $hd(\omega)$. Observe that $\omega\sim_C\omega'$ by Lemma~\ref{alpha iff omega} and the assumption $\alpha\sim_C\alpha'$.  Thus,  $\cK_C\B_C\psi\in hd(\omega')$ by Lemma~\ref{transport lemma}.

Recall that $s(a)=\psi$ for each agent $a\in C$ by the choice of the action profile $s$. Also, $s=_C\delta'$ by the choice of the play  $(\alpha',\delta',\omega')$. Hence, $\delta'(a)=\psi$ for each agent $a\in C$. Thus, $\neg\psi\in hd(\omega')$ by Definition~\ref{canonical play} and because $\cK_C\B_C\psi\in hd(\omega')$. Then,  $\psi\notin hd(\omega')$ the consistency of set $hd(\omega')$. Therefore, $(\alpha',\delta',\omega')\nVdash \psi$ by the induction hypothesis.
\end{proof}

Finally, we are ready to state and prove the strong completeness of our logical system.
\begin{theorem}\label{completeness theorem}
If $X\nvdash\phi$, then there is a game, and a play $(\alpha,\delta,\omega)$ of this game such that $(\alpha,\delta,\omega)\Vdash\chi$ for each $\chi\in X$ and $(\alpha,\delta,\omega)\nVdash\phi$.
\end{theorem}
\begin{proof}
Assume that $X\nvdash\phi$. Hence, set $X\cup\{\neg\phi\}$ is consistent. By Lemma~\ref{Lindenbaum's lemma}, there is a maximal consistent extension $X_0$ of set $X\cup\{\neg\phi\}$. Let game $\left(I,\{\sim_a\}_{a\in\mathcal{A}},\Delta,\Omega,P,\pi\right)$ be the canonical game $G(X_0)$. Also, let $\omega_0$ be the single-element sequence $X_0$. Note that $\omega_0\in \Omega$ by Definition~\ref{canonical outcome}. 
By Lemma~\ref{delta exists lemma},  there is an initial state $\alpha\in I$ and a complete action profile $\delta\in \Delta^\mathcal{A}$ such that $(\alpha,\delta,\omega_0)\in P$. Hence, $(\alpha,\delta,\omega_0)\Vdash\chi$ for each $\chi\in X$ and $(\alpha,\delta,\omega_0)\Vdash\neg\phi$ by Lemma~\ref{induction lemma} and the choice of set $X_0$. Therefore,  $(\alpha,\delta,\omega_0)\nVdash\phi$ by Definition~\ref{sat}.
\end{proof}

\section{Conclusion}\label{conclusion section}

In this article we proposed a definition of blameworthiness in strategic games with imperfect information and gave a sound and complete logical system that captures the interplay between distributed knowledge and blameworthiness modalities. This works extends our previous result for perfect information setting~\cite{nt19aaai}. 

\label{end of article}

\bibliographystyle{elsarticle-num}
\bibliography{sp}

\begin{thebibliography}{10}
\expandafter\ifx\csname url\endcsname\relax
  \def\url#1{\texttt{#1}}\fi
\expandafter\ifx\csname urlprefix\endcsname\relax\def\urlprefix{URL }\fi
\expandafter\ifx\csname href\endcsname\relax
  \def\href#1#2{#2} \def\path#1{#1}\fi

\bibitem{se13eb}
P.~Singer, M.~Eddon, Moral responsibility, problem of, Encyclop{\ae}dia
  BritannicaHttps://www.britannica.com/topic/problem-of-moral-responsibility.

\bibitem{f94p}
L.~Fields, Moral beliefs and blameworthiness: Introduction, Philosophy 69~(270)
  (1994) 397--415.

\bibitem{fr00}
J.~M. Fischer, M.~Ravizza, Responsibility and control: A theory of moral
  responsibility, Cambridge University Press, 2000.

\bibitem{nk07nous}
S.~Nichols, J.~Knobe, Moral responsibility and determinism: The cognitive
  science of folk intuitions, Nous 41~(4) (2007) 663--685.

\bibitem{m15ps}
E.~Mason, Moral ignorance and blameworthiness, Philosophical Studies 172~(11)
  (2015) 3037--3057.

\bibitem{w17}
D.~Widerker, Moral responsibility and alternative possibilities: Essays on the
  importance of alternative possibilities, Routledge, 2017.

\bibitem{f69tjop}
H.~G. Frankfurt, Alternate possibilities and moral responsibility, The Journal
  of Philosophy 66~(23) (1969) 829--839.

\bibitem{c15cop}
F.~Cushman, Deconstructing intent to reconstruct morality, Current Opinion in
  Psychology 6 (2015) 97--103.

\bibitem{h16}
J.~Y. Halpern, Actual causality, MIT Press, 2016.

\bibitem{hk18aaai}
J.~Y. Halpern, M.~Kleiman-Weiner, Towards formal definitions of
  blameworthiness, intention, and moral responsibility, in: Proceedings of the
  Thirty-Second AAAI Conference on Artificial Intelligence (AAAI-18), 2018.

\bibitem{bs18aaai}
V.~Batusov, M.~Soutchanski, Situation calculus semantics for actual causality,
  in: Proceedings of the Thirty-Second AAAI Conference on Artificial
  Intelligence (AAAI-18), 2018.

\bibitem{ahl17aamas}
N.~Alechina, J.~Y. Halpern, B.~Logan, Causality, responsibility and blame in
  team plans, in: Proceedings of the 16th Conference on Autonomous Agents and
  MultiAgent Systems, International Foundation for Autonomous Agents and
  Multiagent Systems, 2017, pp. 1091--1099.

\bibitem{nt19aaai}
P.~Naumov, J.~Tao, Blameworthiness in strategic games, in: Proceedings of
  Thirty-third AAAI Conference on Artificial Intelligence (AAAI-19), 2019.

\bibitem{ali62}
A.~L. Institute, Model Penal Code: Official Draft and Explanatory Notes.
  Complete Text of Model Penal Code as Adopted at the 1962 Annual Meeting of
  the American Law Institute at Washington, D.C., May 24, 1962., The Institute,
  1985 Print.

\bibitem{x98jpl}
M.~Xu, Axioms for deliberative stit, Journal of Philosophical Logic 27~(5)
  (1998) 505--552.

\bibitem{bht09jancl}
J.~Broersen, A.~Herzig, N.~Troquard, What groups do, can do, and know they can
  do: an analysis in normal modal logics, Journal of Applied Non-Classical
  Logics 19~(3) (2009) 261--289.
\newblock \href {http://dx.doi.org/10.3166/jancl.19.261-289}
  {\path{doi:10.3166/jancl.19.261-289}}.

\bibitem{ls11ai}
E.~Lorini, F.~Schwarzentruber, A logic for reasoning about counterfactual
  emotions, Artificial Intelligence 175~(3) (2011) 814.

\bibitem{bp90krdr}
N.~Belnap, M.~Perloff, Seeing to it that: A canonical form for agentives, in:
  Knowledge representation and defeasible reasoning, Springer, 1990, pp.
  167--190.

\bibitem{h01}
J.~F. Horty, Agency and deontic logic, Oxford University Press, 2001.

\bibitem{h95jpl}
J.~F. Horty, N.~Belnap, The deliberative stit: A study of action, omission,
  ability, and obligation, Journal of philosophical logic 24~(6) (1995)
  583--644.

\bibitem{hp17rsl}
J.~Horty, E.~Pacuit, Action types in stit semantics, The Review of Symbolic
  Logic (2017) 1--21.

\bibitem{ow16sl}
G.~K. Olkhovikov, H.~Wansing, Inference as doxastic agency. part i: The basics
  of justification stit logic, Studia Logica (2018) 1--28.

\bibitem{p01illc}
M.~Pauly, Logic for social software, Ph.D. thesis, Institute for Logic,
  Language, and Computation (2001).

\bibitem{p02}
M.~Pauly, A modal logic for coalitional power in games, Journal of Logic and
  Computation 12~(1) (2002) 149--166.
\newblock \href {http://dx.doi.org/10.1093/logcom/12.1.149}
  {\path{doi:10.1093/logcom/12.1.149}}.

\bibitem{g01tark}
V.~Goranko, Coalition games and alternating temporal logics, in: Proceedings of
  the 8th conference on Theoretical aspects of rationality and knowledge,
  Morgan Kaufmann Publishers Inc., 2001, pp. 259--272.

\bibitem{vw05ai}
W.~van~der Hoek, M.~Wooldridge, On the logic of cooperation and propositional
  control, Artificial Intelligence 164~(1) (2005) 81 -- 119.

\bibitem{b07ijcai}
S.~Borgo, Coalitions in action logic, in: 20th International Joint Conference
  on Artificial Intelligence, 2007, pp. 1822--1827.

\bibitem{sgvw06aamas}
L.~Sauro, J.~Gerbrandy, W.~van~der Hoek, M.~Wooldridge, Reasoning about action
  and cooperation, in: Proceedings of the Fifth International Joint Conference
  on Autonomous Agents and Multiagent Systems, AAMAS '06, ACM, New York, NY,
  USA, 2006, pp. 185--192.
\newblock \href {http://dx.doi.org/10.1145/1160633.1160663}
  {\path{doi:10.1145/1160633.1160663}}.

\bibitem{abvs10jal}
T.~{\AA}gotnes, P.~Balbiani, H.~van Ditmarsch, P.~Seban, Group announcement
  logic, Journal of Applied Logic 8~(1) (2010) 62 -- 81.
\newblock \href {http://dx.doi.org/10.1016/j.jal.2008.12.002}
  {\path{doi:10.1016/j.jal.2008.12.002}}.

\bibitem{avw09ai}
T.~{\AA}gotnes, W.~van~der Hoek, M.~Wooldridge, Reasoning about coalitional
  games, Artificial Intelligence 173~(1) (2009) 45 -- 79.
\newblock \href {http://dx.doi.org/10.1016/j.artint.2008.08.004}
  {\path{doi:10.1016/j.artint.2008.08.004}}.

\bibitem{b14sr}
F.~Belardinelli, Reasoning about knowledge and strategies: Epistemic strategy
  logic, in: Proceedings 2nd International Workshop on Strategic Reasoning,
  {SR} 2014, Grenoble, France, April 5-6, 2014, Vol. 146 of {EPTCS}, 2014, pp.
  27--33.

\bibitem{gjt13jaamas}
V.~Goranko, W.~Jamroga, P.~Turrini, Strategic games and truly playable
  effectivity functions, Autonomous Agents and Multi-Agent Systems 26~(2)
  (2013) 288--314.
\newblock \href {http://dx.doi.org/10.1007/s10458-012-9192-y}
  {\path{doi:10.1007/s10458-012-9192-y}}.

\bibitem{alnr11jlc}
N.~Alechina, B.~Logan, H.~N. Nguyen, A.~Rakib, Logic for coalitions with
  bounded resources, Journal of Logic and Computation 21~(6) (2011) 907--937.

\bibitem{ga17tark}
R.~Galimullin, N.~Alechina, Coalition and group announcement logic, in:
  Proceedings Sixteenth Conference on Theoretical Aspects of Rationality and
  Knowledge (TARK) 2017, Liverpool, UK, 24-26 July 2017, 2017, pp. 207--220.

\bibitem{ge18aamas}
V.~Goranko, S.~Enqvist, Socially friendly and group protecting coalition
  logics, in: Proceedings of the 17th International Conference on Autonomous
  Agents and Multiagent Systems, International Foundation for Autonomous Agents
  and Multiagent Systems, 2018, pp. 372--380.

\bibitem{nr18kr}
P.~Naumov, K.~Ros, Strategic coalitions in systems with catastrophic failures
  (extended abstract), in: Proceedings of the 16th International Conference on
  Principles of Knowledge Representation and Reasoning, 2018.

\bibitem{ja07jancl}
W.~Jamroga, T.~{\AA}gotnes, Constructive knowledge: what agents can achieve
  under imperfect information, Journal of Applied Non-Classical Logics 17~(4)
  (2007) 423--475.
\newblock \href {http://dx.doi.org/10.3166/jancl.17.423-475}
  {\path{doi:10.3166/jancl.17.423-475}}.

\bibitem{jv04fm}
W.~Jamroga, W.~van~der Hoek, Agents that know how to play, Fundamenta
  Informaticae 63~(2-3) (2004) 185--219.

\bibitem{v01ber}
J.~van Benthem, Games in dynamic-epistemic logic, Bulletin of Economic Research
  53~(4) (2001) 219--248.
\newblock \href {http://dx.doi.org/10.1111/1467-8586.00133}
  {\path{doi:10.1111/1467-8586.00133}}.

\bibitem{b08deon}
J.~Broersen, A logical analysis of the interaction between `obligation-to-do'
  and `knowingly doing', in: International Conference on Deontic Logic in
  Computer Science, Springer, 2008, pp. 140--154.

\bibitem{nt17aamas}
P.~Naumov, J.~Tao, Coalition power in epistemic transition systems, in:
  Proceedings of the 2017 International Conference on Autonomous Agents and
  Multiagent Systems (AAMAS), 2017, pp. 723--731.

\bibitem{w15lori}
Y.~Wang, A logic of knowing how, in: Logic, Rationality, and Interaction,
  Springer, 2015, pp. 392--405.

\bibitem{w17synthese}
Y.~Wang, A logic of goal-directed knowing how, Synthese (2016) 1--21.

\bibitem{aa16jlc}
T.~{\AA}gotnes, N.~Alechina, Coalition logic with individual, distributed and
  common knowledge, Journal of Logic and ComputationExv085.
\newblock \href {http://dx.doi.org/10.1093/logcom/exv085}
  {\path{doi:10.1093/logcom/exv085}}.

\bibitem{fhlw17ijcai}
R.~Fervari, A.~Herzig, Y.~Li, Y.~Wang, Strategically knowing how, in:
  Proceedings of the Twenty-Sixth International Joint Conference on Artificial
  Intelligence, {IJCAI-17}, 2017, pp. 1031--1038.

\bibitem{nt17tark}
P.~Naumov, J.~Tao, Together we know how to achieve: An epistemic logic of
  know-how, in: 16th conference on Theoretical Aspects of Rationality and
  Knowledge (TARK), July 24-26, 2017, EPTCS 251, 2017, pp. 441--453.

\bibitem{nt18ai}
P.~Naumov, J.~Tao, Together we know how to achieve: An epistemic logic of
  know-how, Artificial Intelligence 262 (2018) 279 -- 300.
\newblock \href
  {http://dx.doi.org/https://doi.org/10.1016/j.artint.2018.06.007}
  {\path{doi:https://doi.org/10.1016/j.artint.2018.06.007}}.

\bibitem{nt18aaai}
P.~Naumov, J.~Tao, Strategic coalitions with perfect recall, in: Proceedings of
  Thirty-Second AAAI Conference on Artificial Intelligence, 2018.

\bibitem{nt18aamas}
P.~Naumov, J.~Tao, Second-order know-how strategies, in: Proceedings of the
  2018 International Conference on Autonomous Agents and Multiagent Systems
  (AAMAS), 2018, pp. 390--398.

\bibitem{h17}
J.~Y. Halpern, Reasoning about uncertainty, MIT press, 2017.

\bibitem{g18foxnews}
M.~Gant, Millennials being blamed for decline of {A}merican cheese, Fox
  NewsHttps://www.foxnews.com/food-drink/millennials-kraft-american-cheese-sales-decline.amp.

\bibitem{m09}
E.~Mendelson, Introduction to mathematical logic, CRC press, 2009.

\bibitem{fhmv95}
R.~Fagin, J.~Y. Halpern, Y.~Moses, M.~Y. Vardi, Reasoning about knowledge, MIT
  Press, Cambridge, MA, 1995.

\end{thebibliography}

\end{document}